\documentclass[twoside]{article}

\usepackage[preprint]{aistats2027}

\usepackage{times}
\usepackage{basicreq}
\usepackage{math_notations}

\usepackage[round]{natbib}

\begin{document}

\runningtitle{Risk-Aware Alignment}
\runningauthor{Banerjee, Saha, Gopalan}

\twocolumn[

\aistatstitle{Risk-Aware Alignment Must Incorporate Reward Uncertainty}



\aistatsauthor{ Debangshu Banerjee \And Kintan Saha \And Aditya Gopalan }
\aistatsaddress{ HP AI Labs \\ HP Inc. \And
                 Undergraduate Programme \\ Indian Institute of Science \And
                 Department of ECE \\ Indian Institute of Science }
]

\begin{abstract}
We consider the problem of policy optimization with noisy rewards, where the policy is constrained to be in the neighborhood of a reference policy.We study the risk of the policy optimization process, defined to be the chance of obtaining a policy worse off than the reference policy.
In an idealized multivariate Gaussian model for the noisy reward function, we show that the regularizer achieving optimal risk in a suitably local neighborhood of the reference policy is the quadratic norm weighted by the reward covariance matrix. 
This result is used to develop a new policy optimization algorithm based on regularization via reward variance estimates. 
Experiments across synthetic bandit models, large language models and reward models show that this algorithm yields more robust alignment outcomes than existing approaches.
\end{abstract}

\section{Introduction}
\label{sec:introduction}
\begin{figure}[!htbp]
    \centering
    \begin{minipage}{0.45\textwidth} 
        \centering
        \begin{tikzpicture}[
            scale = 0.70, 
            transform shape,
            node distance=0.5cm,
            every node/.style={rectangle, draw, align=center, text width=7cm},
            prompt/.style={fill=blue!20},
            response/.style={fill=green!20}
            ]

            \node[prompt] (prompt) {%
            \textbf{Prompt:}\\[5pt]
            \emph{Who created the Superman cartoon character?}
            };

            \node[response, below=of prompt] (response) {%
            \textbf{Response:}\\[0.001pt]
            \emph{Superman, the iconic comic book superhero, was created by writer Jerry Siegel and artist Joe Shuster. Superman first appeared in Action Comics \#1, which was published by Detective Comics, Inc. (later DC Comics) in June 1938.}
            };
        \end{tikzpicture}
    \end{minipage}%
    \hspace{0.04\textwidth} 
    \begin{minipage}{0.45\linewidth}
        \centering
        \includegraphics[width=0.8\linewidth]{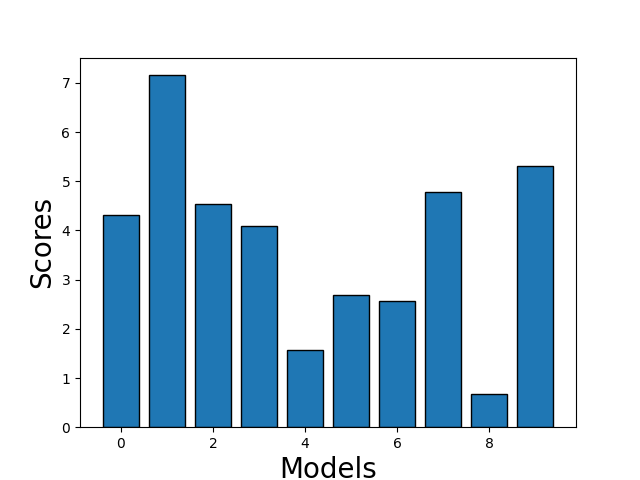} 
    \end{minipage}
    \caption{\footnotesize{Reward scores assigned by $10$ reward models on the same prompt-response pair. The reward models are identical in that they are trained independently on the same dataset, with the same hyperparameters and number of epochs. Despite this, we see a wide variation in the score assigned by each model. }}
    \label{figure:evidence}
\end{figure}

Reinforcement Learning with Human Feedback (RLHF)~\citep{christiano2017deep, ziegler2019fine} is a fine-tuning method for aligning large language models (LLMs) with human preferences. It is a key process in the post-training of systems such as \textsc{ChatGPT}~\citep{openai2023gpt}, \textsc{Claude}~\citep{claude_3}, \textsc{Gemini}~\citep{team2023gemini}, and \textsc{LLaMA-3}~\citep{llama_cite}. RLHF aims to achieve model alignment by guiding LLMs toward human-defined ethical, safety, and utility objectives. In the standard RLHF framework~\citep{ouyang2022training, bai2022constitutional, touvron2023llama}, a reward model predicts human preferences, serving as a proxy for human judgment and providing the reward signal for downstream policy optimization. 

\textbf{Challenges of Reward Model Reliability.}
Reward models in RLHF can be inherently noisy, posing challenges for robust alignment. As our experiments show, ten independently-trained, identically configured reward models assign varying scores to the same prompt-response pair (Figure~\ref{figure:evidence}). This variability stems from three key factors. First, reward models are trained on far less data, typically hundreds of thousands of preference pairs, than LLMs, limiting generalization and introducing statistical noise. Second, stochastic optimization techniques, such as mini-batch gradient descent with few training epochs~\citep{stiennon2020learning, meta2024introducing}, further increase randomness across trained instances. Lastly, in the case where absolute reward scores, instead of relative preferences, are directly elicited from a population (of humans or AI agents) for downstream usage in reward model training or policy optimization, score variability naturally arises due to diversity of opinions and biases. These factors imply that reward models are not reliable oracles, and overfitting to their noisy outputs may lead to alignment with spurious artifacts rather than true human intent.

\begin{figure*}[!t]
\begin{minipage}[b]{0.45\linewidth}
\centering
\begin{tikzpicture}[scale=0.75, every node/.style={scale=0.75}]
\def\armspacing{2.5}

\draw[->, thick] (-1, 0) -- (2*\armspacing+1, 0) node[right]{};
\draw[->, thick] (-1, 0) -- (-1, 4.5) node[above] {Reward};

\foreach \y in {0,1,...,4} {
    \draw (-1.01, \y) -- (-0.9, \y);
    \node[left] at (-1, \y) {\small{\y}};
}

\draw[black] (0*\armspacing, 0) -- (0*\armspacing, 4);
\draw[red] (-0.2+\armspacing*0, 0.5) -- (0.2+\armspacing*0, 0.5);
\draw[red] (-0.2+\armspacing*0, 3.5) -- (0.2+\armspacing*0, 3.5);
\draw[red] (-0.2+\armspacing*0, 0.5) -- (-0.2+\armspacing*0, 0.7);
\draw[red] (0.2+\armspacing*0, 0.5) -- (0.2+\armspacing*0, 0.7);
\draw[red] (-0.2+\armspacing*0, 3.5) -- (-0.2+\armspacing*0, 3.3);
\draw[red] (0.2+\armspacing*0, 3.5) -- (0.2+\armspacing*0, 3.3);
\fill[blue] (0*\armspacing, 3.2) circle (0.07);
\node[blue, right] at (0.2+\armspacing*0, 3.2) {\small{$\hat{R}_1$}};
\fill[green!50!blue] (0*\armspacing, 1.0) circle (0.07);
\node[green!50!blue, left] at (-0.2+\armspacing*0, 1.0) {$r^*_1$};
\node[black] at (0*\armspacing, -0.5) {\footnotesize{Response/Arm 1}};

\draw[black] (1*\armspacing, 0) -- (1*\armspacing, 4);
\draw[red] (-0.2+\armspacing*1, 1.6) -- (0.2+\armspacing*1, 1.6);
\draw[red] (-0.2+\armspacing*1, 2.4) -- (0.2+\armspacing*1, 2.4);
\draw[red] (-0.2+\armspacing*1, 1.6) -- (-0.2+\armspacing*1, 1.7);
\draw[red] (0.2+\armspacing*1, 1.6) -- (0.2+\armspacing*1, 1.7);
\draw[red] (-0.2+\armspacing*1, 2.4) -- (-0.2+\armspacing*1, 2.3);
\draw[red] (0.2+\armspacing*1, 2.4) -- (0.2+\armspacing*1, 2.3);
\fill[blue] (1*\armspacing, 2.2) circle (0.07);
\node[blue, right] at (0.2+\armspacing*1, 2.2) {\small{$\hat{R}_2$}};
\fill[green!50!blue] (1*\armspacing, 1.8) circle (0.07);
\node[green!50!blue, left] at (-0.2+\armspacing*1, 1.8) {$r^*_2$};
\node[black] at (1*\armspacing, -0.5) {\footnotesize{Response/Arm 2}};

\draw[black] (2*\armspacing, 0) -- (2*\armspacing, 4);
\draw[red] (-0.2+\armspacing*2, 1.2) -- (0.2+\armspacing*2, 1.2);
\draw[red] (-0.2+\armspacing*2, 1.8) -- (0.2+\armspacing*2, 1.8);
\draw[red] (-0.2+\armspacing*2, 1.2) -- (-0.2+\armspacing*2, 1.25);
\draw[red] (0.2+\armspacing*2, 1.2) -- (0.2+\armspacing*2, 1.25);
\draw[red] (-0.2+\armspacing*2, 1.8) -- (-0.2+\armspacing*2, 1.75);
\draw[red] (0.2+\armspacing*2, 1.8) -- (0.2+\armspacing*2, 1.75);
\fill[green!50!blue] (2*\armspacing, 1.65) circle (0.07);
\node[green!50!blue, left] at (-0.2+\armspacing*2, 1.65) {$r_3^*$};
\fill[blue] (2*\armspacing, 1.35) circle (0.07);
\node[blue, right] at (0.2+\armspacing*2, 1.35) {\small{$\hat{R}_3$}};
\node[black] at (2*\armspacing, -0.5) {\footnotesize{Response/Arm 3}};

\end{tikzpicture}
\caption{\footnotesize A 3-armed bandit problem illustrating true rewards $r_1^*, r_2^*, r_3^*$ (green circles), estimated rewards $\hat{R}_1, \hat{R}_2, \hat{R}_3$ (blue circles), and their uncertainty intervals (red brackets). Although Arm 1 has the lowest true reward, it yields the highest estimate $\hat{R}_1$, while Arms 2 and 3 have more accurate but lower estimates. A naive policy that relies solely on the estimates would overemphasize Arm 1, resulting in a lower expected return. In contrast, a variance-aware strategy that accounts for the higher uncertainty in $\hat{R}_1$ assigns it a lower probability, leading to a more robust and higher average return. 
}
\label{fig:example}
\end{minipage}
\hfill
\begin{minipage}[b]{0.51\linewidth}
    \centering
    \includegraphics[width=0.85\linewidth]{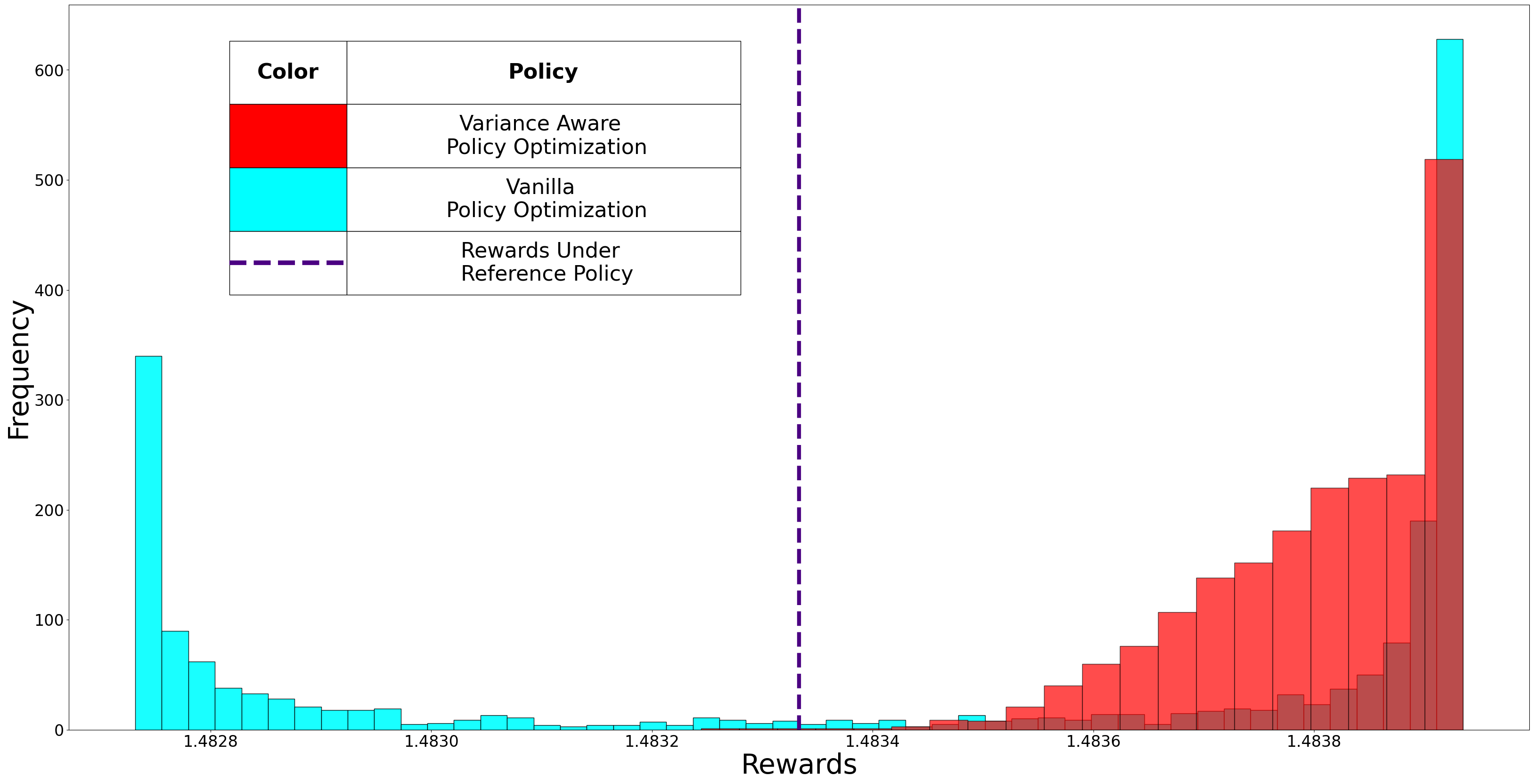}
    \caption{\footnotesize
We plot the expected return $\sum_{i=1}^3 \pi_i r^*_i$ for the 3-armed bandit example shown in Figure \ref{fig:example}, where the probabilities $\{\pi_i\}_{i=1}^3$ are computed based on observed estimates $\{\hat{R}_i\}_{i=1}^3$. We generate 1000 normally distributed samples of $\{\hat{R}_i\}$, and for each sample, derive a corresponding probability vector $\pi$. The cyan-colored distribution represents the returns when $\pi$ is chosen without considering reward uncertainty, while the red distribution accounts for variance information. The initial policy is uniform over the three arms, yielding an average return of $1.4833$. Without using variance information, there is a $39\%$ chance of performing worse than the initial policy, whereas this probability drops to $0.15\%$ when variance-aware probabilities are used. This highlights the reduced risk associated with incorporating uncertainty into policy selection.
}
\label{fig:example-soln}
\end{minipage}
\end{figure*}

\textbf{Example Illustration.} To illustrate the importance of accounting for uncertainty in reward models, consider a simple bandit problem with 3 arms, for which any %
policy (`LLM') assigns probabilities to different arms ('responses') to maximize the expected return (although an LLM represents a contextual bandit policy, we consider a simplified example without contexts to make our case). In this example, the true rewards $r_1^*, r_2^*, r_3^*$ (green circles in Figure~\ref{fig:example}) satisfy $r_1^* < r_2^* < r_3^*$. However, the estimated rewards $\hat{R}_1, \hat{R}_2, \hat{R}_3$ (blue circles) suggest the opposite ranking, with Arm 1 appearing most favorable due to a spurious high estimate. If a policy is optimized purely based on these estimates, it assigns the highest probability to Arm 1, leading to a reduced true return. Crucially, the uncertainty intervals (red brackets) reveal that the reward estimate for Arm 1 is highly uncertain, while Arms 2 and 3 offer more stable, albeit lower, estimates. A more discerning, variance-aware policy improvement strategy ought to recognize this and assign lower weight to the unreliable high estimate of Arm 1, instead favoring the more dependable estimates of Arms 2 and 3. This example suggests a natural notion of \textbf{risk} for policy optimization under reward uncertainty: the probability that optimizing the policy using noisy reward estimates produces a policy whose true expected reward is lower than that of the reference policy. To quantify this risk, Figure~\ref{fig:example-soln} plots the expected return $\sum_{i=1}^3 \pi_i r_i^*$ over $1000$ samples $\hat{R}_i\sim \cN(r^*_i, \sigma_i^2)$, for $i =1,2,3$. Policies learned without variance information (cyan distribution) exhibit a $39\%$ chance of underperforming the initial uniform policy, whereas this probability drops to $0.15\%$ when variance is incorporated (red distribution). The base policy is a uniform distribution over the three arms and is marked as a purple vertical line. These observations motivate the central question of our work:

\textit{Among local policy updates around a reference policy, which updates minimize the probability of degrading the true expected reward relative to that reference?}

\textbf{Contributions.}
We address this question through three contributions. First, we empirically document substantial heterogeneity in reward-model uncertainty across prompt–response pairs. Second, under a Gaussian model of reward uncertainty, we derive an exact characterization of policy-degradation risk and show that the reward covariance-weighted quadratic trust region minimizes this risk among local quadratic trust-region updates for both non-parameterized and parameterized policies. Third, we develop \textbf{Risk-Aware Policy Optimization (RAPO)}, which converts the theoretically risk-optimal covariance geometry into a practical reward transformation that can be integrated into existing policy-optimization pipelines. We validate the theoretical predictions on synthetic bandits and LLM policy-optimization settings.

\textbf{Related Work.} 
Reward models in RLHF approximate human preferences but are vulnerable to failures like \textit{reward hacking}~\citep{amodei2016concrete} and \textit{overoptimization}~\citep{gao2023scaling}, especially under limited data and stochastic training. To address this, prior work has proposed ensemble methods~\citep{coste2023reward, eisenstein2023helping, rame2024warm, zhang2024improving}, conservative objectives using lower confidence bounds (LCBs)~\citep{liang2022reward, zhai2023uncertainty, xiong2024iterative}, and uncertainty modeling via Bayesian dropout~\citep{gal2016dropout} or reparameterization~\citep{lou2024uncertainty}. Recent approaches bypass separate reward models by prompting LLMs directly~\citep{kwon2023reward, llmasajudge2023, singla-etal-2024-dynamic}, or by eliciting uncertainty estimates from LLMs via confidence scores~\citep{xiong2023can}, rationality judgments~\citep{tanneru2023quantifying}, ensemble variance~\citep{tonolini2024bayesian}, or belief updates~\citep{yadkori2024believe}. We build on these insights by explicitly incorporating reward uncertainty into policy optimization to improve robustness.
\section{Variability in Reward Models}
\label{sec:variability}

We first characterize the variability of independently trained reward models. We train an ensemble of ten reward models under identical conditions and measure their disagreement on the same prompt--response pairs. Full training details are deferred to Appendix~\ref{sec:reward_model}.

For each input in the four \textsc{RewardBench} datasets~\citep{lambert2024rewardbench}, we compute the sample variance across the ten ensemble members. Table~\ref{tab:variance_stats} summarizes these variances by their mean and median, with the full distributions shown in Appendix~\ref{sec:reward_model}. The median variance of the reward assigned to the accepted response ranges from $3.81$ on \textsc{Chat} to $5.14$ on \textsc{Safety}, indicating substantial score variation despite identical training conditions. Since rewards lie on a $1$--$20$ scale, a variance of $4$ corresponds to a standard deviation of $2$.

Because reward models are used through the preferences they induce, absolute score variation can partly reflect model-specific offsets. We therefore also examine the variance of the accepted-minus-rejected reward difference, which is invariant to constant offsets. This variance is smaller but remains nonzero, ranging from a mean of $2.37$ on \textsc{Chat} to $0.09$ on \textsc{Reasoning}. Thus, ensemble members can disagree not only in their absolute scores but also in the preferences they induce, although their relative preferences are much more consistent on \textsc{Reasoning}.

Importantly, this variability is heterogeneous across inputs: some prompt--response pairs receive highly consistent scores, while others elicit substantial disagreement. Treating all reward estimates as equally reliable therefore ignores meaningful differences in their uncertainty. This motivates the pair-dependent uncertainty model in Section~\ref{sec:theory}, where each pair is associated with its own variance $\sigma^2(x,y)$.

\begin{table}[t]
\centering
\footnotesize
\setlength{\tabcolsep}{4pt}
\begin{tabular}{lcccc}
\toprule
& \multicolumn{2}{c}{$\V_i\big[r_i(x,y^1)\big]$} & \multicolumn{2}{c}{$\V_i\big[r_i(x,y^1) - r_i(x,y^2)\big]$} \\
\cmidrule(lr){2-3}\cmidrule(lr){4-5}
dataset & mean & median & mean & median \\
\midrule
\textsc{Chat}      & $4.22$ & $3.81$ & $2.37$ & $1.38$ \\
\textsc{Chat-Hard} & $4.95$ & $4.53$ & $0.80$ & $0.24$ \\
\textsc{Safety}    & $5.56$ & $5.14$ & $1.56$ & $0.85$ \\
\textsc{Reasoning} & $4.71$ & $4.44$ & $0.09$ & $0.01$ \\
\bottomrule
\end{tabular}
\caption{\footnotesize Variance of the rewards assigned by $10$ independently trained reward models to the same input. For each input we take the sample variance $\V_i$ across the ten members, of the reward given to the accepted response $y^1$ (left) and of the accepted-minus-rejected difference (right), which would vanish were the members separated only by a constant offset; each entry is the mean or median of that variance over the inputs of a dataset. Rewards lie on a $1$--$20$ scale, so these are variances of scores on that scale. Figures~\ref{fig:reward_modelling} and \ref{fig:reward_diff} in Appendix~\ref{sec:reward_model} plot the full distributions.}
\label{tab:variance_stats}
\end{table}
\section{Preliminaries}
\label{sec:theory}
\textbf{Notations:}
We assume that prompts are strings $x$ from a prompt set $\mathcal{X}$, and responses are strings $y$ from a response set $\mathcal{Y}$\footnote{Typically, $\mathcal{X}$ and $\mathcal{Y}$ both equal the space of all finite-length token sequences in an LLM context, but they can differ depending on the  application.}. A reward model assigns a scalar value to each prompt–response pair, i.e., $R : \mathcal{X} \times \mathcal{Y} \to \mathbb{R}$. We model a large language model (LLM) as a policy $\pi$ that defines a probability distribution over responses given a prompt, i.e., $\pi(\cdot \mid x) \in \cP({\cY})$, and we denote by $\rho$ the distribution over prompts, which is fixed and satisfies $\rho(x) > 0$. With a slight abuse of notation, we treat the policy $\pi$ as an induced joint distribution over $(x, y)$ pairs, $\pi(x,y) = \rho(x)\pi(y \mid x)$, which allows us to write the expected reward compactly as $R^\top \pi$. We write $N := \abs{\cX}\abs{\cY}$ for the dimension of this representation and $\bm{1}$ for the all-ones vector on $\cY$. For any vector $d$ and symmetric positive definite $M$, we define the $M$-weighted norm as $\|d\|_{M} := \sqrt{d^\top M d}$. Finally, $\pi_0$ denotes a reference policy (e.g., from pretraining), which we assume satisfies $\pi_0(y \mid x) > 0$ for all $(x,y)$, and $\Pi_0 := \diag(\pi_0(x,y))$.

For quantities indexed by responses at a fixed prompt we use a subscript $x$: thus $R_x \in \mathbb{R}^{\abs{\cY}}$ are the reward estimates at $x$, and
\[
d_x := \pi(\cdot\mid x) - \pi_0(\cdot\mid x)
\]
is the deviation of the conditional policy. A policy is normalized \emph{at every prompt separately}, $\sum_y\pi(y\mid x) = 1$ for each $x$, so that $\bm{1}^\top d_x = 0$ for each $x$.    

\textbf{Alignment as a Constrained Optimization Problem.}
Reinforcement learning from human feedback~\citep{christiano2017deep,ziegler2019fine,stiennon2020learning,ouyang2022training} uses a reward model $R$ fitted to human comparisons and maximizes the expected reward under that reward model, while holding the policy near the reference policy it starts from,
\begin{equation}
\label{eq:rlhf}
\max_\pi\ R^\top\pi
\quad \text{s.t.} \quad
\sum_x \rho(x)\,\KL\big(\pi(\cdot\mid x)\,\|\,\pi_0(\cdot\mid x)\big) \leq \delta .
\end{equation}
where $\KL(\cdot ||\cdot)$ is the KL-divergence \emph{trust region}. There are several reasons for wanting to keep the finetuned policy close to the reference policy. First is the issue of \emph{reward hacking}: optimizing a proxy reward $R$ can lead to overoptimization, where the policy exploits errors in the proxy and ultimately degrades with respect to the true reward~\citep{amodei2016concrete}. This proxy--true reward gap grows as the policy moves farther from the reference~\citep{gao2023scaling,coste2023reward}. Second, $\pi_0$ embodies substantial investment in pretraining and supervised fine-tuning~\citep{radford2019language,ouyang2022training}, so large deviations from it risk discarding capabilities that the alignment reward does not capture. Such degradation, often referred to as the alignment tax, has been observed empirically~\citep{ouyang2022training}. However, while using the KL divergence as the constraint is standard practice~\citep{stiennon2020learning,ouyang2022training}, there is no universal reason that it should be the appropriate notion of distance from the reference policy.

\subsection{Mathematical Model}


Section~\ref{sec:variability} shows that independently trained reward models can assign
substantially different scores to the same prompt-response pair. Towards formalizing this, we model the reward available to policy optimization by positing that there exists an underlying
true reward function $r^*$, while the reward model available to the
policy optimizer provides a noisy estimate $R$ of it.

\begin{restatable}[\textbf{Gaussian reward model}]{assumption}{noisyreward}
\label{assm:gaussian}
The observed reward is a Gaussian perturbation of the true reward: $R \sim \cN(r^*,\Sigma)$, where $\Sigma$ is a (symmetric, positive definite) covariance matrix on $\mathbb{R}^N$. Thus, for any deviation $d:= \pi -\pi_0$, the estimated improvement in expected reward $R^\top d$ is distributed as a Gaussian with mean $r^{*\top}d$ and variance $\|d\|_\Sigma^2 := d^\top\Sigma d$.
\end{restatable}

We choose a Gaussian model for noisy rewards as it expresses the first and second order moments of the distribution, is tractable for calculations, and can serve as a good approximation for distributions that are sufficiently light tailed (e.g., bounded reward distributions). We impose no additional structure on $\Sigma$. In particular, components of the reward estimate $R$ may be correlated across prompt-response pairs. Let $\pi(R)$ denote the policy returned by the RLHF problem in Equation~\ref{eq:rlhf}. Since $R$ is random under Assumption~\ref{assm:gaussian}, the optimized policy $\pi(R)$ is also random. This induces a probability distribution $\mathbb{P}_\pi := \cN(r^*,\Sigma)\circ\pi^{-1}$ over policies. 

We define the risk of a policy optimization update as the probability (over the randomness of the observed reward and that of the update procedure) that the resulting policy fails to improve over the reference policy under the true reward. It represents the cost of performing policy optimization without obtaining any improvement relative to the starting policy that is attempted to be improved. 

\begin{definition}[\textbf{Risk}]
\label{def:risk}
The {risk} of a policy optimization update with input policy $\pi_0$ and output policy $\pi(R)$ is
\begin{equation}
\label{eq:risk}
\begin{aligned}
\mathrm{Risk}(\pi)
&:= \mathbb{P}_\pi\big[\,\pi' : \pi'^\top r^* \leq \pi_0^\top r^*\,\big] \\
&\phantom{:}= \prob{\pi(R)^\top r^* \leq \pi_0^\top r^*}.
\end{aligned}
\end{equation}
\end{definition}
To ease notation, we write $\pi$ for $\pi(R)$ hereafter. 

\textbf{Trust Regions and Their Geometries.}
\label{sec:geometry} 
We consider three standard notions of policy nearness, and observe that each locally reduces to a weighted Euclidean distance of the form $\|\pi -\pi_0\|_H^2$, differing only in the weighting matrix $H$. First, an ordinary Euclidean trust region $\|d\|_2^2$ gives $H=I$, corresponding to the plain policy gradient. Second, the KL constraint in~\eqref{eq:rlhf} has the local approximation
\begin{equation}
\label{eq:bregman_quadratic}
\begin{aligned}
&\sum_x\rho(x)\KL\big(\pi(\cdot\mid x)\,\|\,\pi_0(\cdot\mid x)\big) \\
&\qquad\approx
\tfrac12\|\pi - \pi_0\|^2_{\Pi_0^{-1}} ,
\end{aligned}
\end{equation}
so $H=\Pi_0^{-1}$. Third, the natural policy gradient~\citep{kakade2001natural} uses the same $\Pi_0^{-1}$ weighting in parameter space: for $J=\partial\pi/\partial\theta$, its Fisher information satisfies $F(\theta)=J^\top\Pi_0^{-1}J$. Thus, these methods differ only in their choice of $H$, none of which depends on the reward model.

\textbf{The general problem.}
We study the regularized policy optimization problem for a weighted $\ell_2$ metric where the weight matrix $H$ is symmetric positive definite on $\mathbb{R}^N$. Let
\begin{equation}
\label{eq:trust_region}
\pi_H = \arg\max_\pi\ R^\top\pi
\quad \text{s.t.} \quad
\tfrac12\|\pi - \pi_0\|_H^2 \leq \epsilon_H ,
\end{equation}
subject to $\sum_y\pi(y\mid x) = 1$ and $\pi(y\mid x) \geq 0$ at every prompt. We call $H$ the \emph{geometry} of the update. It decides which deviations are costly and so controls the direction in which the optimizer responds to the noisy reward, while the radius $\epsilon_H \in \mathbb{R}$ controls the trust region size. 

\section{Risk Optimal Geometry}
\label{sec:optimal}

In this section we prove that the reward covariance  $\Sigma$ is the {\em risk-optimal} geometry among all local geometries $H$. We show that the risk depends on $H$ only through the angle, under the $\Sigma$-inner product $\langle u,v\rangle_\Sigma:=u^\top\Sigma v$, between the direction induced by $H$ from the true reward $r^*$ and a fixed direction $y^*$. Risk increases with this angle. The covariance geometry induces $y^*$ itself, giving angle zero, which is optimal by Cauchy--Schwarz.

\textbf{Admissible Deviations.}
We represent a policy by its joint probabilities $\pi(x,y)=\rho(x)\pi(y\mid x)\in\mathbb{R}^N$. Let $e_x \in \mathbb{R}^N$ be the indicator of prompt $x$, $e_x(x',y) := \mathbf{1}\{x' = x\}$, so that the normalization constraints in~\eqref{eq:trust_region} are $e_x^\top\pi=\rho(x)$ for every $x$. 

Writing the update as $d:=\pi-\pi_0$, both policies satisfy the same prompt-wise normalization, so $e_x^\top d=0$. Define $E:=[\,e_x\,]_{x\in\cX}$ and for a matrix $A$, let $\ker A := \{v: Av = 0\}$ denote the kernel of $A$. Then, updates redistribute mass among responses within each prompt:
\begin{equation}
\label{eq:tangent}
\cT := \big\{\,d\in\mathbb{R}^N : e_x^\top d=0\ \text{for every }x\,\big\}
= \ker E^\top,
\end{equation}
Its orthogonal complement, $\cT^\perp=\mathrm{span}\{e_x\}$, consists of vectors that are constant across responses within each prompt. Equivalently, $d\in\cT$ means $\bm{1}^\top d_x=0$ for every $x$.

Writing $\pi=\pi_0+d$ with $d\in\cT$, the reward objective becomes $R^\top\pi
=R^\top\pi_0+R^\top d$, so only the restriction of $R$ to $\cT$ matters. In particular, $R$ and $R+n$ induce the same objective for every $n\in\cT^\perp$. Thus, reward vectors are identifiable only up to a prompt-wise constant. To represent the resulting linear functional as an update direction, we need to specify an inner product on $\cT$.

\begin{definition}[\textbf{Centered representative}]
\label{def:riesz}
For symmetric positive definite $H$ and $z\in\mathbb{R}^N$, let $\cR_H(z)$ be the unique element of $\cT$ satisfying
\begin{equation}
\label{eq:riesz}
\big\langle \cR_H(z),u\big\rangle_H=z^\top u
\qquad\text{for every }u\in\cT.
\end{equation}
\end{definition}

Such a representation exists and is unique thanks to the Riesz representation theorem~\citep{hoffman1971linear}. Lemma~\ref{lem:riesz_form} gives
\begin{equation}
\label{eq:riesz_closed}
\begin{aligned}
\cR_H(z) &= H^{-1}\big(z - E\lambda_H(z)\big) , \\
\text{where } \lambda_H(z) &= A_H^{-1}E^\top H^{-1}z \;,\;\; A_H := E^\top H^{-1}E ,
\end{aligned}
\end{equation}
Thus, $\lambda_H(z)$ provides one baseline per prompt. These baselines are analogous to the baselines used for variance reduction in policy-gradient methods and decouple across prompts exactly when $e_{x'}^\top H^{-1}e_x=0$ for $x'\neq x$.

To compare geometries, we restrict attention to updates local enough that the positivity constraints in~\eqref{eq:trust_region} are inactive, i.e., no action probability becomes $0$.

\begin{definition}[\textbf{Interior radius}]
\label{def:interior}
Let $q_{\min} := \min_{x,y}\pi_0(x,y) > 0$ be the smallest reference probability over pairs. The radius $\epsilon_H$ is called interior if $2\epsilon_H < \lambda_{\min}(H)\,q_{\min}^2$.
\end{definition}

Lemma~\ref{lem:coupling} shows that an interior radius forces every feasible policy to be strictly positive, so the constraint $\pi \geq 0$ holds with strict inequality and its multipliers vanish. The trust-region constraint is then almost surely active and the solution of~\eqref{eq:trust_region} is
\begin{equation}
\label{eq:interior_solution_main}
\pi_H - \pi_0 = \tfrac{1}{\beta}\,\cR_H(R) ,
\end{equation}
with $\beta > 0$ fixed by the radius $\epsilon_H$. Thus $H$ sets the update direction through $\cR_H$ and the radius sets only its magnitude. Henceforth $\epsilon_H$ is taken to be interior.

\begin{restatable}[\textbf{Optimality of the Covariance Geometry}]{theorem}{optgeom}
\label{thm:opt_geometry}
Let Assumption~\ref{assm:gaussian} hold, let $\abs{\cY}\geq 2$, and suppose $r^*_x$ is non-constant for at least one prompt. Let $H$ be a symmetric positive definite matrix on $\mathbb{R}^N$. Then for every $\epsilon_H$ and $\epsilon_\Sigma$ that are interior in the sense of Definition~\ref{def:interior} for $H$ and for $\Sigma$ respectively, chosen independently of one another, the solutions $\pi_H$ and $\pi_\Sigma$ of~\eqref{eq:trust_region} at those radii satisfy
\[
\mathrm{Risk}(\pi_\Sigma) \;\leq\; \mathrm{Risk}(\pi_H) ,
\]
with equality if and only if the deviation of $\pi_H$ is a positive multiple of that of $\pi_\Sigma$.
\end{restatable}

\begin{proof}
By Lemma~\ref{lem:coupling} the risk event for each $H$ reduces to the half-space $\{R^\top v_H \leq 0\}$, where
\begin{equation}
\label{eq:vH}
v_H := \cR_H(r^*) \in \cT\setminus\{0\}
\end{equation}
is deterministic. The radius does not appear, consistent with its setting the length of the update but not its direction in~\eqref{eq:interior_solution_main}. Under Assumption~\ref{assm:gaussian} the probability of such a half-space has a closed form: for every deterministic $v \in \cT\setminus\{0\}$, the scalar $R^\top v$ is normal with mean $v^\top r^*$ and variance $\|v\|_\Sigma^2$, so
\begin{equation}
\label{eq:angle}
\prob{R^\top v \leq 0} = \Phi\big(-z_{\max}\cos\theta_\Sigma(v, y^*)\big) ,
\end{equation}
where $\Phi$ is the standard normal distribution function, $\theta_\Sigma$ is the angle induced by $\langle\cdot,\cdot\rangle_\Sigma$, and
\begin{equation}
\label{eq:ystar}
y^* := \cR_\Sigma(r^*) , \qquad z_{\max} := \|y^*\|_\Sigma .
\end{equation}
The identity $v^\top r^* = \langle y^*, v\rangle_\Sigma$ needed for~\eqref{eq:angle} is Definition~\ref{def:riesz} applied to $z = r^*$ and $u = v$.

Hence $\mathrm{Risk}(\pi_H) = \Phi(-z_{\max}\cos\theta_\Sigma(v_H,y^*))$. Since $\Phi$ is increasing and $z_{\max} > 0$, the risk is minimized by maximizing $\cos\theta_\Sigma(v_H,y^*)$, so by Cauchy--Schwarz the minimum is $\Phi(-z_{\max})$, attained exactly when $v_H$ is a positive multiple of $y^*$. Taking $H = \Sigma$ gives $v_\Sigma = y^*$ by~\eqref{eq:ystar}, which is therefore the minimizer. The full proof is given in Appendix~\ref{app:proof_optgeom}.
\end{proof}

\begin{remark}
\label{rem:implication}
Insofar as the risk of policy optimization is concerned, the theorem shows, rather surprisingly, that the optimal local regularization geometry in {\em policy} space must be driven by the epistemic uncertainty (covariance $\Sigma$ in our case) in the {\em reward} space. This establishes a principled basis for the notion of `nearness' for policy optimization. 
\end{remark}

\begin{remark}
\label{rem:duality}
The random \emph{policy update} $\pi_H-\pi_0=\tfrac{1}{\beta}\,\cR_H(R)$ differs from the deterministic \emph{induced direction} $v_H:=\cR_H(r^*)$. By the self-adjointness of $\cR_H$, $z^\top\cR_H(w)=w^\top\cR_H(z)$, and hence $r^{*\top}(\pi_H-\pi_0) =\tfrac{1}{\beta}\,R^\top v_H$. Thus, the true gain of the random update has the same sign as the observed reward evaluated on the fixed direction $v_H$, making the risk event $\{r^{*\top}(\pi_H-\pi_0)\leq0\}$ equivalent to the Gaussian half-space $\{R^\top v_H\leq0\}$.
\end{remark}

\begin{remark}
\label{rem:boundary}
Once the radius leaves the interior regime, the simplex boundary becomes active: the positivity constraints bind, \eqref{eq:interior_solution_main} no longer applies, and the ordering of Theorem~\ref{thm:opt_geometry} is no longer guaranteed. Appendix~\ref{app:prelim} gives the corresponding truncated solution~\eqref{eq:clipped}, and Section~\ref{sec:experiment} shows that the covariance geometry can indeed become worse on the boundary.
\end{remark}

\subsection{Parameterized policies.}
An LLM is not optimized over the probability simplex directly but over its parameters, so its updates are confined to the directions its parameterization can realize. Let $\pi_\theta, \enspace \theta \in \mathbb{R}^d$ be such a parameterization, $\theta_0$ be the parameters of the reference policy, and
\[
J := \frac{\partial\pi_\theta}{\partial\theta}\bigg|_{\theta_0} \in \mathbb{R}^{N \times d}
\]
be the Jacobian of the induced distribution over pairs, so that a parameter step $\delta\theta$ moves the policy to first order by $J\delta\theta$. The analogue of~\eqref{eq:trust_region} is a trust region in parameter space, measured by a symmetric positive definite metric $M$ and its own radius $\epsilon_M > 0$:
\begin{equation}
\label{eq:trust_region_param}
\max_{\delta\theta}\ R^\top\big(\pi_0 + J\delta\theta\big)
\quad \text{s.t.} \quad
\tfrac12\,\delta\theta^\top M\,\delta\theta \leq \epsilon_M .
\end{equation}
Because $\pi_\theta(\cdot\mid x)$ is normalized for every $\theta$, differentiating $\sum_y \pi_\theta(y\mid x)=1$ gives $\mathbf 1^\top J_x=0$ for every prompt $x$. Hence any first-order perturbation $J_x\delta\theta$ preserves normalization. Positivity, however, is not automatic: for a sufficiently large step, some coordinates of $\pi_0+J\delta\theta$ may become negative. We therefore impose a radius condition ensuring that the linearized policy remains inside the simplex. Writing $\varrho(J) := \max_{x,y}\|\nabla_\theta\pi_\theta(x,y)\|_2$, a step feasible for~\eqref{eq:trust_region_param} moves any one coordinate by at most $\varrho(J)\sqrt{2\epsilon_M/\lambda_{\min}(M)}$, so the linearized policy remains in the simplex whenever
$2\epsilon_M < \lambda_{\min}(M)\,q_{\min}^2 \big/ \varrho(J)^2$.

Solving~\eqref{eq:trust_region_param} gives $\delta\theta = \frac{1}{\beta}M^{-1}J^\top R$ with $\beta > 0$ set by $\epsilon_M$, so the metric acts as a \emph{preconditioner} on the policy gradient. The policy itself moves by $J\delta\theta$, so the attainable deviations are exactly the range of $J$, a subspace of $\mathbb{R}^N$ of dimension $d$ when $J$ has full column rank.

\begin{restatable}{proposition}{parametric}
\label{prop:parametric}
Let $J = \partial\pi_\theta/\partial\theta$ have full column rank at $\theta_0$, and let $\pi_M$ solve~\eqref{eq:trust_region_param} for a symmetric positive definite metric $M$ with its own radius $\epsilon_M > 0$ satisfying the interior point condition. Then the induced direction is $v_M = JM^{-1}J^\top r^*$, which does not depend on $\epsilon_M$, its risk is given by~\eqref{eq:angle}, and $M = J^\top\Sigma J$ minimizes that risk over all such metrics and radii.
\end{restatable}

By~\eqref{eq:angle} the risk of an update is set by its $\Sigma$-angle to $y^*$, so the safest reachable direction is the $\Sigma$-orthogonal projection of $y^*$ onto the range of $J$, which is the direction $M = J^\top\Sigma J$ induces. The parameterization therefore enters only through how well its range approximates $y^*$, and the risk attains the non-parametrized optimum $\Phi(-z_{\max})$ exactly when $y^*$ is itself reachable. The full proof is available in Appendix~\ref{app:proof_parametric}.


\section{Risk-Aware Policy Optimization (RAPO)}
\label{sec:algorithm}

\textbf{Score Function.} Theorem~\ref{thm:opt_geometry} shows that the $\Sigma$-geometry is risk-optimal, but directly imposing a trust region weighted by the reward covariance is impractical in standard policy-optimization pipelines. In contrast, these pipelines can readily replace the optimization objective. We therefore define a \emph{score function} that captures the same risk-optimal direction while remaining compatible with standard trust-region geometries. Define,
\begin{equation}
\label{eq:precondition}
\begin{aligned}
S &:= \cR_\Sigma(R) = \Sigma^{-1}\big(R - E\hat{\lambda}\big) , \\
\text{where}\quad \hat\lambda &= A_\Sigma^{-1}E^\top \Sigma^{-1},
\quad A_\Sigma:=E^\top \Sigma^{-1}E,
\end{aligned}
\end{equation}
so that $\pi_\Sigma = \pi_0 + \beta^{-1} S$ and $\bm{1}^\top S_x = 0$ at every prompt. 

In the experiments, we estimate $\Sigma$ from an ensemble, with one reward per prompt--response pair and no estimate of the correlation between pairs. Thus, we use the diagonal covariance approximation $\Sigma = \textrm{diag}(\sigma^2(x, y))$. In that case $A_\Sigma$ is diagonal, the baselines decouple, and~\eqref{eq:precondition} is the inverse-variance weighting
\begin{equation}
\label{eq:precondition_diag}
\begin{aligned}
S(x,y) &= \frac{R(x,y) - c_x}{\sigma^2(x,y)} , \\
c_x &= \frac{\sum_y R(x,y)/\sigma^2(x,y)}{\sum_y 1/\sigma^2(x,y)} ,
\end{aligned}
\end{equation}
a reward divided by its variance and centered, within each prompt, by the variance-weighted mean reward. This is the form used throughout Sections~\ref{sec:experiment} and~\ref{sec:llm}.

\subsection{Family of Risk Optimal Geometries} 
\label{sec:rapo}
Theorem~\ref{thm:opt_geometry} identifies the covariance-based trust region geometry, but a typical policy optimization pipeline imposes the KL trust region geometry. The following proposition shows that this costs nothing provided that the score function, given by Equation~\ref{eq:precondition}, is reweighted appropriately. Thus any geometry attains the optimal risk when it carries the correspondingly weighted score.

\begin{restatable}[\textbf{Geometry--Objective Equivalence}]{proposition}{scoreform}
\label{prop:score_form}
Let Assumption~\ref{assm:gaussian} hold, let $S$ be the score~\eqref{eq:precondition}, let $G$ be symmetric positive definite on $\mathbb{R}^N$, and let $\epsilon$ be interior for $G$ in the sense of Definition~\ref{def:interior}. Let $\pi_G$ solve
\begin{equation}
\label{eq:score_objective}
\max_\pi\ \big(GS\big)^\top\pi
\quad \text{s.t.} \quad
\tfrac12\|\pi - \pi_0\|_G^2 \leq \epsilon ,
\end{equation}
over the policies of~\eqref{eq:trust_region}. Then, almost surely, $\pi_G - \pi_0 = \eta_G\,S$ with step length $\eta_G = \sqrt{2\epsilon}/\|S\|_G > 0$, a positive multiple of the deviation of $\pi_\Sigma$ for every such $G$. Consequently $\mathrm{Risk}(\pi_G) = \mathrm{Risk}(\pi_\Sigma) = \Phi(-z_{\max})$, for every such $G$ and every pair of interior radii.
\end{restatable}

Proposition~\ref{prop:score_form} yields a simple geometry--objective matching rule: for any quadratic trust-region geometry $G$, the risk-optimal objective is $w = GS$, and no further centering is needed because $S \in \cT$ already. Consider the following three cases. $G = \Sigma$ gives $w = R - E\hat{\lambda}$, which is equivalent to $R$ since $E\hat{\lambda} \in \cT^\perp$, and recovers Theorem~\ref{thm:opt_geometry}. $G = I$ gives $w = S$, so the score is used unchanged under an $\ell_2$ trust region. Also, $G = \Pi_0^{-1}$ yields the KL geometry of~\eqref{eq:bregman_quadratic} with $\Pi_0 := \diag(\pi_0(x,y))$, gives $w(x,y) = S(x,y)/\pi_0(x,y)$. Full proof details are provided in Appendix~\ref{app:proof_scoreform}.

Neither of the last two is available in a policy-gradient pipeline over sequences, for the reasons given in Appendix~\ref{sec:implementability}: the $\ell_2$ constraint underflows single precision, and dividing the score by $\pi_0(x,y)$ gives an objective of unbounded variance. For a practical algorithm we therefore prescribe \textbf{RAPO}, which pairs the raw score $S$ with the KL geometry that PPO already imposes. This forfeits the equality of Proposition~\ref{prop:score_form}, and Section~\ref{sec:candidate_experiment} measures what the substitution costs on an action space where every pairing can be solved exactly.

\begin{remark}
    Equal risk does not imply identical policies: $\pi_G - \pi_0$ and $\pi_\Sigma - \pi_0$ are positive multiples of $S$ with step lengths $\eta_G$ and $\beta^{-1}$ that agree only if $\epsilon/\epsilon_\Sigma = \|S\|_G^2/\|S\|_\Sigma^2$, which is a ratio depending on $R$.
\end{remark}

\section{Bandit Experiments}
\label{sec:bandits}
\label{sec:experiment}

We explore the theory's implications on synthetic bandit problems, where $r^*$ and $\Sigma$ are known and the constrained problem is solved to optimality. Every quantity the theory predicts can then be checked against its measured counterpart, including in regimes which the theorems exclude.

\textbf{Methodology.}
A bandit with $K$ arms stands for one prompt with $K$ candidate responses. We fix $(\pi_0, r^*, \Sigma)$ with $\Sigma$ diagonal, its variances drawn log-uniformly over a range of width $\kappa$, draw $R \sim \cN(r^*, \Sigma)$, and solve~\eqref{eq:trust_region} exactly, positivity included, through the KKT system of Lemma~\ref{lem:kkt}. The comparison covers the four geometries of Section~\ref{sec:geometry}, vanilla ($H = I$), covariance ($H = \Sigma$), KL ($H = \Pi_0^{-1}$) and variance-weighted KL ($H = \Pi_0^{-1}\Sigma$), all scored by $R$, together with the uncertainty penalty UWO$(\lambda)$, which scores $R - \lambda\sigma^2$ in the KL geometry ($H = \Pi_0^{-1}$). Risk is reported both in the closed form of~\eqref{eq:angle} and by Monte Carlo over sampled $R$, averaged over independent instances of $(\pi_0, r^*, \Sigma)$. Appendix~\ref{sec:bandits_supp} gives the parameter values and the normalization that holds the difficulty fixed as $K$ grows.

\textbf{Observations.}
The covariance geometry is the safest throughout, underperforming on $5$ to $7\%$ of reward models against $16$ to $19\%$ for the unweighted baseline, a gap stable from $K = 50$ to $K = 1000$ (Table~\ref{tab:bandit_arms}); UWO$(0.1)$ merely matches that baseline and UWO$(1.0)$ roughly doubles its risk, consistent with its shifting the score rather than reshaping the trust region. Closed-form and Monte-Carlo estimates agree to within $0.004$, which validates the Gaussian reduction of Section~\ref{sec:geometry}. Two pairs of rows coincide exactly, because a uniform $\pi_0$ makes $\Pi_0^{-1} = KI$ and the solution of~\eqref{eq:trust_region} depends on $H$ only up to positive scaling. Setting the reference distribution as $\pi_0(i) \propto (i+1)^{-\alpha}$ for $i \in \{0, 1,\dots K-1\}$ leaves the $\ell_2$ and covariance geometries untouched while both KL geometries decay with $\alpha$ (the skew), until at $\alpha = 2$ the variance-weighted KL ($0.374$) is worse than the unweighted baseline ($0.194$) (Table~\ref{tab:bandit_skew}). Increasing the radius past the interior threshold reproduces Remark~\ref{rem:boundary} quantitatively (Table~\ref{tab:bandit_radius}): the closed form is exact up to $\epsilon_0$ and badly wrong beyond it, predicting $0.035$ against a measured $0.099$, and at $200\epsilon_0$ the covariance geometry is the riskier of the two on one instance in twenty, an ordering Theorem~\ref{thm:opt_geometry} rules out inside the interior regime. The remaining sweeps are provided in Appendix~\ref{sec:bandits_supp}.

\begin{table}[t]
\centering
\footnotesize
\begin{tabular}{lcccc}
\toprule
& \multicolumn{4}{c}{number of arms $K$} \\
\cmidrule(lr){2-5}
method & $50$ & $100$ & $500$ & $1000$ \\
\midrule
vanilla       & $0.163$ & $0.190$ & $0.180$ & $0.185$ \\
covariance    & $\bm{0.055}$ & $\bm{0.067}$ & $\bm{0.051}$ & $\bm{0.058}$ \\
UWO$(0.1)$    & $0.169$ & $0.198$ & $0.181$ & $0.176$ \\
UWO$(1.0)$    & $0.291$ & $0.330$ & $0.263$ & $0.197$ \\
KL            & $0.163$ & $0.190$ & $0.180$ & $0.185$ \\
var.-wtd.\ KL & $\bm{0.055}$ & $\bm{0.067}$ & $\bm{0.051}$ & $\bm{0.058}$ \\
\bottomrule
\end{tabular}
\caption{\footnotesize Underperformance probability $\prob{\pi^\top r^* \leq \pi_0^\top r^*}$ for a uniform reference policy at $\kappa = 50$, at each geometry's interior radius. With $\pi_0$ uniform, KL coincides with the unweighted geometry and variance-weighted KL with the covariance geometry.}
\label{tab:bandit_arms}
\end{table}

\begin{table}[t]
\centering
\footnotesize
\setlength{\tabcolsep}{3.5pt}
\begin{tabular}{lccccc}
\toprule
& \multicolumn{5}{c}{skew $\alpha$} \\
\cmidrule(lr){2-6}
method & $0$ & $0.5$ & $1$ & $1.5$ & $2$ \\
\midrule
vanilla        & $0.194$ & $0.194$ & $0.194$ & $0.194$ & $0.194$ \\
covariance     & $\bm{0.058}$ & $\bm{0.058}$ & $\bm{0.058}$ & $\bm{0.058}$ & $\bm{0.058}$ \\
KL             & $0.194$ & $0.216$ & $0.311$ & $0.382$ & $0.414$ \\
var.-wtd.\ KL  & $0.058$ & $0.087$ & $0.190$ & $0.302$ & $0.374$ \\
\bottomrule
\end{tabular}
\caption{\footnotesize Underperformance probability at $K = 100$ against the skew of the reference, $\pi_0 \propto (i+1)^{-\alpha}$, for which $p_{\min}$ falls from $10^{-2}$ at $\alpha = 0$ to $6.1\cdot10^{-5}$ at $\alpha = 2$.}
\label{tab:bandit_skew}
\end{table}

\begin{table}[t]
\centering
\footnotesize
\setlength{\tabcolsep}{4pt}
\begin{tabular}{lccccc}
\toprule
& \multicolumn{5}{c}{radius $\epsilon / \epsilon_0$} \\
\cmidrule(lr){2-6}
& $0.25$ & $1$ & $4$ & $25$ & $200$ \\
\midrule
vanilla, measured    & $0.035$ & $0.035$ & $0.032$ & $0.033$ & $0.099$ \\
vanilla, predicted   & $0.035$ & $0.035$ & $0.035$ & $0.035$ & $0.035$ \\
covariance, measured & $0.008$ & $0.007$ & $0.008$ & $0.009$ & $0.028$ \\
covariance, predicted& $0.008$ & $0.008$ & $0.008$ & $0.008$ & $0.008$ \\
\midrule
on boundary, vanilla    & $0.00$ & $0.00$ & $0.17$ & $0.97$ & $1.00$ \\
on boundary, covariance & $0.00$ & $0.00$ & $0.03$ & $0.65$ & $1.00$ \\
instances reversed      & $0/20$ & $0/20$ & $0/20$ & $0/20$ & $1/20$ \\
\bottomrule
\end{tabular}
\caption{\footnotesize Behaviour across the interior threshold $\epsilon_0$ of Definition~\ref{def:interior}. ``Measured'' is the simulated risk, obtained by solving the constrained problem exactly for each sampled $R$; ``predicted'' is the closed form $\Phi(-z)$, which is derived assuming $\pi \geq 0$ holds with strict inequality and is therefore constant in the radius. ``On boundary'' is the fraction of draws whose optimum has $\pi(x,y) = 0$ for at least one pair. ``Reversed'' counts instances in which the covariance geometry is the riskier of the two by more than three standard errors of the difference, so that sampling noise is not recorded as a violation.}
\label{tab:bandit_radius}
\end{table}

\section{LLM Experiments}
\label{sec:llm}

The experiments in this section use the \textsc{AlpacaFarm} instruction dataset \citep{dubois2023alpacafarm}. Unless stated otherwise the reward model is a twenty-member \textsc{Llama-3.2-1B} ensemble \citep{llama_cite}, whose training is detailed in Appendix~\ref{sec:reward_model}, and the \textsc{AlpacaFarm} human reward model is used as the true reward $r^*$ and is not used for training. Section~\ref{sec:candidate_experiment} draws candidates from a supervised \textsc{Qwen2.5-0.5B} policy; Section~\ref{sec:ppo_experiment} fine-tunes supervised \textsc{Pythia-1.4B} \citep{biderman2023pythia} and \textsc{Qwen2.5-7B} \citep{Yang2024Qwen25TR} policies with PPO. The uncertainty penalty $R - \lambda\sigma^2$ of \citet{coste2023reward} at $\lambda = 0.1$, their best-performing setting, is used throughout as a baseline. Appendix~\ref{sec:llm_supp} gives the remaining details.

\subsection{Exact Tests on Candidate Sets}
\label{sec:candidate_experiment}

Section~\ref{sec:algorithm} prescribes how to weight the score $\hat{S}$ against a given constraint so that the risk is optimal. We verify the three pairings it declares optimal: the covariance trust region with $R$, by Theorem~\ref{thm:opt_geometry}; the $\ell_2$ trust region with $\hat{S}$, by Proposition~\ref{prop:score_form}; and the KL trust region with $\hat{S}/(\rho\pi_0)$. Against these, we compare three practically realizable pairings: $R$ under KL, which is vanilla PPO; $R$ under $\ell_2$, which is the vanilla policy gradient; and $\hat{S}$ under KL, our algorithm \textbf{RAPO} detailed in Section~\ref{sec:algorithm}. These comparisons quantify the cost of replacing the theoretically optimal geometry–objective pairings with those that are implementable in standard policy-optimization pipelines. Solving all six exactly requires a finite action space: keeping $K$ candidates per prompt puts $p_{\min}$ at order $1/K$, so Definition~\ref{def:interior} is satisfiable and the closed form~\eqref{eq:interior_solution_main} applies, the policy being a vector in $\Delta^{K-1}$.

\textbf{Methodology.}
For each of $512$ prompts we generate $K = 16$ candidates. The reference policy anneals their sequence probabilities $p_j$ as $\pi_0(y_j \mid x) \propto p_j^{\,\log\alpha/\Delta_x}$ with $\Delta_x := \max_j \log p_j - \min_j \log p_j$, so that the most and least likely candidate differ by a factor $\alpha$, the skewness of the reference; $\alpha = 1$ is uniform. Each pairing is solved exactly, and we report the angle its deviation $\pi - \pi_0$ makes with that of $\pi_\Sigma$ together with its probability of underperformance.

\begin{table}[t!]
\centering
\footnotesize
\setlength{\tabcolsep}{1.9pt}
\begin{tabular}{lcccccc}
\toprule
& \multicolumn{3}{c}{Direction Cosine} & \multicolumn{3}{c}{$\prob{\text{underperform}}$} \\
\cmidrule(lr){2-4}\cmidrule(lr){5-7}
Geometry, Score & $\alpha=1$ & $\alpha=5$ & $\alpha=20$ & $\alpha=1$ & $\alpha=5$ & $\alpha=20$ \\
\midrule
Covariance, $R$            & $1.000$ & $1.000$ & $1.000$ & $\bm{0.235}$ & $\bm{0.236}$ & $\bm{0.234}$ \\
$\ell_2$, $\hat{S}$              & $1.000$ & $1.000$ & $1.000$ & $\bm{0.235}$ & $\bm{0.236}$ & $\bm{0.234}$ \\
KL, $\hat{S}/\rho\pi_0$          & $1.000$ & $1.000$ & $1.000$ & $\bm{0.235}$ & $\bm{0.236}$ & $\bm{0.234}$ \\
\midrule
$\ell_2$, $R$              & $0.763$ & $0.764$ & $0.764$ & $0.292$ & $0.293$ & $0.291$ \\
KL, $R$                    & $0.763$ & $0.720$ & $0.643$ & $0.292$ & $0.302$ & $0.318$ \\
\textbf{RAPO}                   & $1.000$ & $0.934$ & $0.820$ & $\bm{0.235}$ & $\bm{0.249}$ & $\bm{0.273}$ \\
KL, UWO$(0.1)$             & $0.702$ & $0.702$ & $0.702$ & $0.310$ & $0.310$ & $0.308$ \\
\bottomrule
\end{tabular}
\caption{\footnotesize Update direction and risk on candidate sets, by the skewness factor $\alpha$ of the reference policy. The first column group is the cosine between a geometry's update direction and that of $\pi_\Sigma$, the optimum of Theorem~\ref{thm:opt_geometry}; the three pairings above the rule are the optimal pairing described in Section~\ref{sec:algorithm} and are seen to equal the prediction. The second group is the fraction of draws on which the policy is worse than the reference under the gold reward, in bold for the three optimal pairings and for \textbf{RAPO}.}
\label{tab:candidate}
\end{table}

\textbf{Observations.}
The three optimal pairings have $\cos = 1$ to machine precision at every $\alpha$ and their risks agree to three decimals, so the prescription of Section~\ref{sec:algorithm} is exact on real prompts, candidates and gold reward. The mismatched pairings are riskier, as expected. The $\ell_2$ rows do not depend on $\alpha$ while the KL rows grow with it, since only the KL geometry involves $\pi_0$. \textbf{RAPO} recovers most of the difference: it is exact at $\alpha = 1$, where a uniform $\pi_0$ makes $\Pi_0^{-1}$ a multiple of the identity, and at $\alpha = 20$ it closes $0.045$ of the $0.084$ separating vanilla PPO from the optimum. 

\subsection{PPO Experiments}
\label{sec:ppo_experiment}
The optimal pairing of $\hat{S}$ with an $\ell_2$ trust region is not implementable at sequence scale: the $\ell_2$ constraint involves sequence probabilities that underflow, while the interior trust-region radius becomes vanishingly small. The alternative pairing of $\hat{S}/(\rho\pi_0)$ with the KL geometry is compatible with the standard PPO pipeline, but its inverse-$\pi_0$ reweighting has prohibitively high variance at sequence scale. We therefore use the practical pairing of $\hat{S}$ with KL (\textbf{RAPO}). As shown in the previous section, RAPO is not risk-optimal but nevertheless improves over vanilla PPO, which pairs $R$ with KL. For details, see Appendix~\ref{sec:implementability}.

\textbf{Observations} We detail the methodology of this experiment in Appendix~\ref{sec:llm_supp} and present the observations here. We compare RAPO with vanilla PPO and the uncertainty penalty UWO$(0.1)$, both of which use $R$ as the objective, on \textsc{Pythia-1.4B} and \textsc{Qwen2.5-7B} across a range of trust-region radii. Table~\ref{tab:ppo_policy_scaling} reports the probability that the resulting policy underperforms the reference under the gold reward. RAPO consistently reduces this risk, with the largest gains on \textsc{Qwen2.5-7B}: its risk is $4$--$14\times$ lower than both baselines across all radii. On \textsc{Pythia-1.4B}, RAPO has the lowest risk at four of the five radii, with all methods performing similarly at the tightest radius. In contrast, UWO$(0.1)$ is essentially indistinguishable from vanilla PPO across both models. An ablation over reward-model ensemble size on \textsc{Qwen2.5-7B} is provided in Appendix~\ref{sec:reward_scale}.

\begin{table}[t!]
\centering
\footnotesize
\setlength{\tabcolsep}{3.5pt}
\begin{tabular}{lccccc}
\toprule
& \multicolumn{5}{c}{trust-region radius} \\
\cmidrule(lr){2-6}
Objective & $0.5$ & $1$ & $2$ & $4$ & $8$ \\
\midrule
\multicolumn{6}{c}{\textbf{\textsc{Pythia-1.4B}}} \\
Vanilla & $0.155$ & $0.162$ & $0.202$ & $0.267$ & $0.378$ \\
UWO$(0.1)$ & $0.157$ & $0.163$ & $0.203$ & $0.263$ & $0.372$ \\
\textbf{RAPO} & $\bm{0.148}$ & $\bm{0.163}$ & $\bm{0.187}$ & $\bm{0.215}$ & $\bm{0.247}$ \\
\midrule
\multicolumn{6}{c}{\textbf{\textsc{Qwen2.5-7B}}} \\
Vanilla & $0.277$ & $0.363$ & $0.472$ & $0.553$ & $0.643$ \\
UWO$(0.1)$ & $0.267$ & $0.338$ & $0.462$ & $0.535$ & $0.623$ \\
\textbf{RAPO} & $\bm{0.070}$ & $\bm{0.055}$ & $\bm{0.058}$ & $\bm{0.052}$ & $\bm{0.047}$ \\
\bottomrule
\end{tabular}
\caption{\footnotesize Probability that the fine-tuned policy underperforms the reference policy under the gold reward, across trust-region radii and objectives. Experimental details are provided in Appendix~\ref{sec:ppo_details}.}
\label{tab:ppo_policy_scaling}
\end{table}

\clearpage
\subsection*{AI use statement}
In this work, we used generative AI tools solely for proofreading the manuscript and for suggestions to improve the writing style and clarity.
We have \textbf{not} used generative AI tools to develop theoretical analysis or conceptual frameworks, formulate mathematical claims, provide critical ingredients for proving mathematical claims, assist in the writing of proofs, propose or refine hypotheses, design or provide feedback on research methodology or experiments. 
We take responsibility for the final content of this work.


\bibliography{References}

\section*{Checklist}

\begin{enumerate}

  \item For all models and algorithms presented, check if you include:
  \begin{enumerate}
    \item A clear description of the mathematical setting, assumptions, algorithm, and/or model. [Yes]
    \item An analysis of the properties and complexity (time, space, sample size) of any algorithm. [Yes]
    \item (Optional) Anonymized source code, with specification of all dependencies, including external libraries. [No]
  \end{enumerate}

  \item For any theoretical claim, check if you include:
  \begin{enumerate}
    \item Statements of the full set of assumptions of all theoretical results. [Yes]
    \item Complete proofs of all theoretical results. [Yes]
    \item Clear explanations of any assumptions. [Yes]
  \end{enumerate}

  \item For all figures and tables that present empirical results, check if you include:
  \begin{enumerate}
    \item The code, data, and instructions needed to reproduce the main experimental results. [Yes]
    \item All the training details. [Yes]
    \item A clear definition of the specific measure or statistics and error bars. [Yes]
    \item A description of the computing infrastructure used. [Yes]
  \end{enumerate}

  \item If you are using existing assets or curating/releasing new assets, check if you include:
  \begin{enumerate}
    \item Citations of the creator if your work uses existing assets. [Yes]
    \item The license information of the assets, if applicable. [Not Applicable]
    \item New assets either in the supplemental material or as a URL, if applicable. [Not Applicable]
    \item Information about consent from data providers/curators. [Not Applicable]
    \item Discussion of sensitive content if applicable, e.g., personally identifiable information or offensive content. [Not Applicable]
  \end{enumerate}

  \item If you used crowdsourcing or conducted research with human subjects, check if you include:
  \begin{enumerate}
    \item The full text of instructions given to participants and screenshots. [Not Applicable]
    \item Descriptions of potential participant risks, with links to Institutional Review Board (IRB) approvals if applicable. [Not Applicable]
    \item The estimated hourly wage paid to participants and the total amount spent on participant compensation. [Not Applicable]
  \end{enumerate}

\end{enumerate}

\clearpage
\appendix
\numberwithin{theorem}{subsection}
\thispagestyle{empty}
\onecolumn
\aistatstitle{Risk-Aware Alignment Must Incorporate Reward Uncertainty: \\
Supplementary Materials}

\section{Proofs}
\label{sec:Proofs}

Assumption~\ref{assm:gaussian} is in force throughout, and we restate it for reference.

\noisyreward*

Appendix~\ref{app:prelim} collects the three lemmas the main proofs rest on. Lemma~\ref{lem:kkt} solves the trust region prompt by prompt, which is the form the bandit experiments of Section~\ref{sec:experiment} solve numerically and the only place a coordinatewise argument is needed. Lemma~\ref{lem:riesz_form} gives the closed form of the centred representative of Definition~\ref{def:riesz}. Lemma~\ref{lem:coupling} then solves the trust region for a general metric and identifies the underperformance event as a half-space with a deterministic boundary; this is the engine of everything that follows, and it is where the Gaussian assumption and the interior radius are used. Appendices~\ref{app:proof_optgeom},~\ref{app:proof_parametric} and~\ref{app:proof_scoreform} prove Theorem~\ref{thm:opt_geometry}, Proposition~\ref{prop:parametric} and Proposition~\ref{prop:score_form} in turn. A reader interested only in why the covariance geometry is optimal can read Lemma~\ref{lem:coupling} and then Appendix~\ref{app:proof_optgeom}.

\subsection{Preliminaries}
\label{app:prelim}

Because $\pi_0$ lies in the interior of the simplex at every prompt, every point of a small enough trust region around it is strictly positive, so that the condition $\pi(y\mid x) \geq 0$ holds with strict inequality and its multipliers vanish. The following lemma makes this precise and records the resulting first-order conditions. It is stated for arbitrary symmetric positive definite blocks. Diagonality is needed nowhere, because the step length is controlled through the smallest eigenvalue of $H_x$ rather than coordinatewise; the one place it would be needed is the truncated solution~\eqref{eq:clipped} outside the interior regime.

\begin{lemma}[\textbf{Interior trust region}]
\label{lem:kkt}
Let Assumption~\ref{assm:gaussian} hold, let $H_x$ be symmetric positive definite for every prompt, and assume $\abs{\cY} \geq 2$. Consider
\begin{equation}
\label{eq:constrained}
\max_\pi\ R^\top\pi
\quad \text{s.t.} \quad
\tfrac12\sum_x\rho(x)\|d_x\|_{H_x}^2 \leq \epsilon,
\ \ \sum_y\pi(y\mid x) = 1,
\ \ \pi(y\mid x) \geq 0 \ \ \forall x,y ,
\end{equation}
where $d_x = \pi(\cdot\mid x) - \pi_0(\cdot\mid x)$, $\pi_0(y\mid x) > 0$ and $p_{\min} := \min_{x,y}\pi_0(y\mid x)$. If
\[
2\epsilon < \Big(\min_x \rho(x)\lambda_{\min}(H_x)\Big)\,p_{\min}^2 ,
\]
then every feasible policy is strictly positive, and almost surely the solution of~\eqref{eq:constrained} satisfies
\begin{equation}
\label{eq:interior_solution}
d_x = \tfrac{1}{\beta}H_x^{-1}\big(R_x - c_x\bm{1}\big),
\qquad
c_x = \frac{\bm{1}^\top H_x^{-1}R_x}{\bm{1}^\top H_x^{-1}\bm{1}} .
\end{equation}
Here each $c_x$ is the centering that normalization at prompt $x$ forces, and the single multiplier $\beta$ is obtained by substituting $d_x$ into the active trust-region constraint, which gives
\begin{equation}
\label{eq:beta_value}
\beta = \Big(\frac{1}{2\epsilon}\sum_x\rho(x)\,\big\|R_x - c_x\bm{1}\big\|^2_{H_x^{-1}}\Big)^{1/2} > 0 .
\end{equation}
Only $\beta$ depends on the radius, and it enters $d_x$ as an overall positive factor, so the radius scales the step without changing its direction.
\end{lemma}

\begin{proof}
\textbf{Strict positivity.}
Let $\pi$ be feasible and fix a prompt $x$. Every term of the constraint is non-negative, so the term at $x$ is bounded by the whole sum, and the Rayleigh bound $d_x^\top H_x d_x \geq \lambda_{\min}(H_x)\|d_x\|_2^2$, valid for any symmetric positive definite $H_x$, gives
\[
\rho(x)\,\lambda_{\min}(H_x)\,\|d_x\|_2^2
\;\leq\;
\rho(x)\,d_x^\top H_x d_x
\;\leq\;
\sum_{x'}\rho(x')\,d_{x'}^\top H_{x'} d_{x'}
\;\leq\; 2\epsilon .
\]
Since $\|d_x\|_\infty \leq \|d_x\|_2$, the hypothesis on $\epsilon$ yields
\[
\|d_x\|_\infty^2 \;\leq\; \frac{2\epsilon}{\rho(x)\lambda_{\min}(H_x)} \;<\; p_{\min}^2 ,
\]
so $\|d_x\|_\infty < p_{\min}$ \emph{strictly}. As $\pi_0(y\mid x) \geq p_{\min}$ for every $y$,
\[
\pi(y\mid x) = \pi_0(y\mid x) + d_x(y) \geq p_{\min} - \|d_x\|_\infty > 0 .
\]
Every feasible policy is therefore strictly positive at every prompt, and the trust region is contained in the interior of the product of simplices.

\textbf{First-order conditions.}
The problem maximizes a linear function over a compact convex set and Slater's condition holds at $\pi_0$, so a solution exists and satisfies the KKT conditions: there are multipliers $\beta \geq 0$, one $\lambda_x \in \mathbb{R}$ per prompt, and $\mu \geq 0$ with
\begin{equation}
\label{eq:kkt_stationarity}
\rho(x)R_x - \beta\,\rho(x)H_x d_x - \lambda_x\bm{1} + \mu_x = 0
\qquad \text{for every } x,
\end{equation}
together with $\mu_x(y)\pi(y\mid x) = 0$ and $\beta\big(\tfrac12\sum_x\rho(x)\|d_x\|_{H_x}^2 - \epsilon\big) = 0$. There is one normalization multiplier per prompt, in place of the single multiplier a non-contextual problem would carry. By the first step $\pi(\cdot\mid x) > 0$, so complementary slackness gives $\mu_x = 0$ for every $x$. Dividing~\eqref{eq:kkt_stationarity} by $\rho(x) > 0$ and writing $c_x := \lambda_x/\rho(x)$ leaves
\begin{equation}
\label{eq:kkt_reduced}
R_x - c_x\bm{1} = \beta\,H_x d_x .
\end{equation}
Were $\beta = 0$, \eqref{eq:kkt_reduced} would force $R_x = c_x\bm{1}$ at every prompt, that is $R \in \cV$ where
\[
\cV := \big\{ R \in \mathbb{R}^N \;:\; R_x \in \mathrm{span}(\bm{1}) \ \text{ for every } x \big\}
\]
is a linear subspace of dimension $\abs{\cX}$. Since $\abs{\cY} \geq 2$ we have $\abs{\cX} < N$, so $\cV$ is a proper subspace, and under Assumption~\ref{assm:gaussian} the vector $R$ is Gaussian with nonsingular covariance $\Sigma$ and therefore assigns zero mass to every proper linear subspace of $\mathbb{R}^N$. Hence $\prob{R\in\cV} = 0$, so $\beta > 0$ almost surely, the trust-region constraint is active, and inverting $H_x$ in~\eqref{eq:kkt_reduced} gives the stated form of $d_x$. Finally normalization at prompt $x$ forces $\bm{1}^\top d_x = 0$, so
\[
\bm{1}^\top H_x^{-1}\big(R_x - c_x\bm{1}\big) = 0
\quad\Longrightarrow\quad
c_x = \frac{\bm{1}^\top H_x^{-1}R_x}{\bm{1}^\top H_x^{-1}\bm{1}} ,
\]
the denominator being positive because $H_x^{-1} \succ 0$.
\end{proof}


Symmetry of $H_x$ is used twice, and in both places it is free rather than restrictive. Differentiating the quadratic form in~\eqref{eq:kkt_stationarity} returns $\tfrac12(H_x + H_x^\top)$, and Step 2 of the proof of Theorem~\ref{thm:opt_geometry} uses the symmetry of $H_x^{-1}$ to put the update direction in the form~\eqref{eq:vH}. Since the constraint sees only the symmetric part of $H_x$, nothing is lost by taking $H_x = H_x^\top$ from the outset.

Three further features of the lemma are worth noting. First, the radius condition is deterministic: it bounds the length of the step, which the trust region fixes, whereas $R$ only sets its direction, so a single bound on $\epsilon$ covers every realization of the reward model at once. Second, the problem decouples across prompts except through the single multiplier $\beta$: each prompt receives its own centering $c_x$, determined by normalization at that prompt alone, and $\beta$ is the only quantity that couples them. Third, both conclusions of the lemma require $\epsilon$ to be small, and they are lost one after the other as $\epsilon$ grows. Once $\epsilon$ is large enough for the optimum to fall on the boundary, the condition $\pi(y\mid x) \geq 0$ holds with equality in some coordinates and their multipliers are no longer zero. When in addition $H_x$ is diagonal, so that~\eqref{eq:kkt_stationarity} decouples coordinatewise, the solution takes the truncated form
\begin{equation}
\label{eq:clipped}
\pi(y\mid x) = \Big[\,\pi_0(y\mid x) + \tfrac{1}{\beta}\,\tfrac{R(x,y)-c_x}{H_x(y,y)}\Big]_+ ,
\end{equation}
obtained by reading~\eqref{eq:kkt_stationarity} coordinatewise: where $\pi(y\mid x) > 0$ complementary slackness gives $\mu_x(y) = 0$ and the bracket is the solution, while where $\pi(y\mid x) = 0$ stationarity gives $\mu_x(y) = \lambda_x - \rho(x)R(x,y) - \beta\rho(x)H_x(y,y)\pi_0(y\mid x)$, and $\mu_x(y) \geq 0$ says exactly that the bracket is non-positive. For a non-diagonal $H_x$ the system no longer decouples and no such closed form is available, which is the one place positive definiteness alone is not enough. Enlarging $\epsilon$ further sets more coordinates to zero, until the trust region is large enough to contain a maximizer of $R^\top\pi$ over the feasible set. That maximizer is almost surely a vertex at every prompt, and from then on the trust-region constraint holds with strict inequality at the optimum, so $\beta = 0$ and the solution is that vertex whatever the value of $\epsilon$. Which coordinates are set to zero in~\eqref{eq:clipped} depends on $R$, and that is what breaks the analysis outside the interior regime: the update direction is no longer a fixed vector, so the underperformance event is no longer a half-space in $R$.

\begin{lemma}[\textbf{Closed form of the centred representative}]
\label{lem:riesz_form}
Let $H$ be symmetric positive definite and set $A_H := E^\top H^{-1}E$. Then $A_H$ is symmetric positive definite and $\cR_H$ of Definition~\ref{def:riesz} is given by~\eqref{eq:riesz_closed}. If moreover $e_{x'}^\top H^{-1}e_x = 0$ whenever $x' \neq x$, then $A_H$ is diagonal and $\lambda_H(z)_x = e_x^\top H^{-1}z\,/\,e_x^\top H^{-1}e_x$; in particular this holds when $H$ is block diagonal across prompts, where it reduces to $\lambda_H(z)_x = \bm{1}^\top H_x^{-1}z_x/\bm{1}^\top H_x^{-1}\bm{1}$.
\end{lemma}

\begin{proof}
$A_H$ is the Gram matrix of the columns of $E$ in $\langle\cdot,\cdot\rangle_{H^{-1}}$, hence symmetric positive semidefinite. It is definite because $E$ has full column rank, its columns indicating disjoint nonempty blocks: if $A_H\mu = 0$ then $(E\mu)^\top H^{-1}(E\mu) = 0$, so $E\mu = 0$ and $\mu = 0$.

Condition~\eqref{eq:riesz} says $(H\cR_H(z) - z)^\top u = 0$ for all $u \in \cT$, that is $H\cR_H(z) - z \in \cT^\perp$. By~\eqref{eq:tangent}, $\cT = \ker E^\top$, so $\cT^\perp = \mathrm{range}(E)$ and $H\cR_H(z) = z - E\lambda$ for some $\lambda \in \mathbb{R}^{\abs{\cX}}$. Hence $\cR_H(z) = H^{-1}(z - E\lambda)$, and imposing $E^\top\cR_H(z) = 0$ gives $E^\top H^{-1}z = A_H\lambda$, which determines $\lambda$ uniquely as $A_H \succ 0$. This is~\eqref{eq:riesz_closed}.

Since $(A_H)_{x'x} = e_{x'}^\top H^{-1}e_x$, the stated hypothesis makes $A_H$ diagonal with positive entries $e_x^\top H^{-1}e_x$, and $\lambda = A_H^{-1}E^\top H^{-1}z$ is then read off coordinatewise. A block diagonal $H$ has block diagonal inverse, so $e_{x'}^\top H^{-1}e_x = 0$ for $x' \neq x$, while $e_x^\top H^{-1}e_x = \bm{1}^\top H_x^{-1}\bm{1}$ and $e_x^\top H^{-1}z = \bm{1}^\top H_x^{-1}z_x$.
\end{proof}

Theorem~\ref{thm:opt_geometry} is stated for a general metric, and the following lemma is what carries it. We first fix the representation, because a trust region that couples prompts cannot be written prompt by prompt and the constraints have to be handled together.

\paragraph{Setting.}
Section~\ref{sec:optimal} established the representation this lemma uses: a policy is the vector in $\mathbb{R}^N$ of joint masses $\pi(x,y) = \rho(x)\pi(y\mid x)$, the $\abs{\cX}$ normalization constraints of~\eqref{eq:trust_region} read
\begin{equation}
\label{eq:joint_normalization}
\sum_y \pi(x,y) = \rho(x) ,
\qquad\text{equivalently}\qquad
e_x^\top \pi = \rho(x) ,
\end{equation}
and the admissible deviations are the subspace $\cT = \ker E^\top$ of~\eqref{eq:tangent}, of dimension $\abs{\cX}(\abs{\cY}-1)$, with orthogonal complement $\cT^\perp = \mathrm{span}\{e_x\}$ consisting of the vectors constant in $y$ within each prompt. A trust region that couples prompts cannot be written prompt by prompt, so the constraints are handled together below.
\begin{lemma}[\textbf{Interior solution for a general metric}]
\label{lem:coupling}
Let Assumption~\ref{assm:gaussian} hold, let $H$ be symmetric positive definite on $\mathbb{R}^N$, and suppose $\abs{\cY}\geq 2$ and that $r^*_x$ is non-constant for at least one prompt. Let $\pi_H$ solve~\eqref{eq:trust_region} at a radius with $2\epsilon_H < \lambda_{\min}(H)\,q_{\min}^2$. Then every feasible $\pi$ is strictly positive, so the constraint $\pi \geq 0$ is inactive, and almost surely the solution is unique and satisfies
\[
\pi_H - \pi_0 = \tfrac{1}{\beta}\,\cR_H(R) ,
\]
with $\beta > 0$ fixed by the radius and $\cR_H$ as in Definition~\ref{def:riesz}. Moreover the underperformance event is the half-space
\[
\big\{\pi_H^\top r^* \leq \pi_0^\top r^*\big\} = \big\{R^\top v_H \leq 0\big\} ,
\qquad
v_H := \cR_H(r^*) \in \cT \setminus\{0\} ,
\]
whose boundary is deterministic.
\end{lemma}

\begin{proof}
\textbf{The representative is well posed.}
Since $H \succ 0$, the form $\langle\cdot,\cdot\rangle_H$ is an inner product on $\cT$ and $u \mapsto z^\top u$ is a linear functional on it, so Definition~\ref{def:riesz} determines $\cR_H(z)$ uniquely by the Riesz representation theorem, and Lemma~\ref{lem:riesz_form} gives it in closed form. The map $z \mapsto \cR_H(z)$ is linear and is not injective: only the action of $z$ on $\cT$ enters~\eqref{eq:riesz}, so $\cR_H(z + n) = \cR_H(z)$ for every $n \in \cT^\perp$. Since $\cT^\perp$ is spanned by the prompt indicators, this says that shifting the reward by a constant within each prompt leaves the update unchanged, which is the coordinate-free form of the per-prompt centering.

\textbf{Strict positivity.}
Let $\pi$ be feasible and $d = \pi - \pi_0$. As in Lemma~\ref{lem:kkt}, $\lambda_{\min}(H)\|d\|_2^2 \leq d^\top Hd \leq 2\epsilon$ and $\|d\|_\infty \leq \|d\|_2$, so the hypothesis gives $\|d\|_\infty < q_{\min}$. Hence $\pi(x,y) = \pi_0(x,y) + d(x,y) \geq q_{\min} - \|d\|_\infty > 0$ at every pair, and since $\rho(x) > 0$ this is equivalent to every conditional being strictly positive.

\textbf{The solution.}
Each normalization constraint $e_x^\top\pi = \rho(x)$ of~\eqref{eq:joint_normalization} has gradient $e_x$ and carries its own multiplier $\lambda_x \in \mathbb{R}$. By the previous step the constraint $\pi \geq 0$ is inactive and its multipliers vanish, so stationarity of the Lagrangian reads
\[
R - \beta H d - n = 0 ,
\qquad
n := \sum_x \lambda_x e_x ,
\]
for some $\beta \geq 0$. Thus $n$ is nothing but the $\abs{\cX}$ multipliers packaged into one vector: it takes the value $\lambda_x$ at every response to prompt $x$, that is $n(x,y) = \lambda_x$, so $n$ is constant in $y$ within each prompt. So $n \in \mathrm{span}\{e_x\} = \cT^\perp$.

That is what makes the multipliers disappear. Pairing stationarity with any $u \in \cT$ gives
\[
n^\top u = \sum_x \lambda_x \sum_y u(x,y) = 0 ,
\]
each inner sum vanishing by the definition of $\cT$, and leaves
\[
\beta\,\langle d, u\rangle_H = R^\top u
\qquad \text{for every } u \in \cT .
\]
Were $\beta = 0$ this would force $R^\top u = 0$ for all $u \in \cT$, that is $R \in \cT^\perp$; as $\cT^\perp$ has dimension $\abs{\cX} < N$ and $R$ is Gaussian with nonsingular covariance, that event is null. Hence $\beta > 0$ almost surely, and comparing with~\eqref{eq:riesz} gives $d = \frac{1}{\beta}\cR_H(R)$.

\textbf{The direction is deterministic.}
The representation is self-adjoint: for any $z, w$, applying~\eqref{eq:riesz} twice and using the symmetry of $\langle\cdot,\cdot\rangle_H$,
\[
z^\top \cR_H(w)
= \big\langle \cR_H(z), \cR_H(w)\big\rangle_H
= w^\top \cR_H(z) .
\]
Hence $r^{*\top}d = \frac{1}{\beta}r^{*\top}\cR_H(R) = \frac{1}{\beta}R^\top\cR_H(r^*) = \frac{1}{\beta}R^\top v_H$, and since $\beta > 0$ the underperformance event is $\{R^\top v_H \leq 0\}$. By construction $v_H \in \cT$, so it keeps every prompt normalized. It is non-zero: if $\cR_H(r^*) = 0$ then~\eqref{eq:riesz} gives $r^{*\top}u = 0$ for all $u \in \cT$, so $r^*_x$ is constant in $y$ at every prompt, contrary to hypothesis.

\end{proof}

\subsection{Proof of Theorem~\ref{thm:opt_geometry}}
\label{app:proof_optgeom}
\optgeom*

\begin{proof}
\textbf{Step 1: the solution and the risk event.}
Lemma~\ref{lem:coupling} applies verbatim: it assumes only that $H$ is symmetric positive definite on $\mathbb{R}^N$, and its proof uses $\Sigma$ only through the nonsingularity of the covariance of $R$, which Assumption~\ref{assm:gaussian} provides. It gives $\pi_H - \pi_0 = \beta^{-1}\cR_H(R)$ almost surely with $\beta > 0$, strict positivity of every feasible policy under Definition~\ref{def:interior}, and
\[
\big\{\pi_H^\top r^* \leq \pi_0^\top r^*\big\} = \big\{R^\top v_H \leq 0\big\} ,
\qquad v_H = \cR_H(r^*) \in \cT\setminus\{0\} .
\]
That $v_H \neq 0$ uses the hypothesis that $r^*_x$ is non-constant at some prompt: were $\cR_H(r^*) = 0$, then~\eqref{eq:riesz} would give $r^{*\top}u = 0$ for all $u \in \cT$, forcing $r^*_x$ constant in $y$ at every prompt.

\textbf{Step 2: the risk of a direction.}
Let $v \in \cT\setminus\{0\}$ be deterministic. By Assumption~\ref{assm:gaussian}, $R^\top v$ is normal with mean $v^\top r^*$ and variance $v^\top\Sigma v$, so
\[
\prob{R^\top v \leq 0} = \Phi\big(-z(v)\big) ,
\qquad
z(v) := \frac{v^\top r^*}{\|v\|_\Sigma} .
\]
Applying~\eqref{eq:riesz} with $H = \Sigma$, $z = r^*$ and $u = v$ gives $v^\top r^* = \langle y^*, v\rangle_\Sigma$ directly from the definition of $y^*$, whence
\[
z(v) = \frac{\langle y^*,v\rangle_\Sigma}{\|v\|_\Sigma} = z_{\max}\cos\theta_\Sigma(v,y^*) ,
\]
which is~\eqref{eq:angle}. Here $z_{\max} = \|y^*\|_\Sigma > 0$ because $y^* \neq 0$, by the argument of Step 1 applied to $H = \Sigma$. This is the only place where the diagonal proof had to split the inner product prompt by prompt and cancel the per-prompt baselines; with $y^*$ defined by~\eqref{eq:riesz} the identity is immediate and no structure on $\Sigma$ is used.

\textbf{Step 3: optimality.}
Combining Steps 1 and 2, $\mathrm{Risk}(\pi_H) = \Phi(-z_{\max}\cos\theta_\Sigma(v_H,y^*))$ for every admissible $H$ and radius. Cauchy--Schwarz in $\langle\cdot,\cdot\rangle_\Sigma$ gives
\[
\langle v_H, y^*\rangle_\Sigma \leq \|v_H\|_\Sigma\|y^*\|_\Sigma ,
\qquad\text{i.e.}\qquad
\cos\theta_\Sigma(v_H,y^*) \leq 1 ,
\]
with equality exactly when $v_H$ is a positive multiple of $y^*$. Taking $H = \Sigma$ gives $v_\Sigma = \cR_\Sigma(r^*) = y^*$ by definition, so $\cos\theta_\Sigma = 1$ and $\mathrm{Risk}(\pi_\Sigma) = \Phi(-z_{\max})$. As $\Phi$ is increasing and $z_{\max} > 0$, the inequality $\mathrm{Risk}(\pi_\Sigma) \leq \mathrm{Risk}(\pi_H)$ follows, with equality precisely when $v_H \propto y^*$ with a positive constant, which by Step 1 is the statement that the deviation of $\pi_H$ is a positive multiple of that of $\pi_\Sigma$. Neither bound involves a radius, so the comparison holds for radii chosen independently.
\end{proof}

\subsection{Proof of Proposition~\ref{prop:parametric}}
\label{app:proof_parametric}
\parametric*
\begin{proof}
Since each $\pi_\theta(\cdot\mid x)$ is a probability mass function, $\sum_y\pi_\theta(y\mid x) = 1$ for all $\theta$, and differentiating gives $\bm{1}^\top J_x = 0$ for every prompt; the parameterized update therefore preserves normalization at each prompt on its own and~\eqref{eq:trust_region_param} carries no multiplier for it.

\textbf{Optimal update for a given metric.}
The Lagrangian of~\eqref{eq:trust_region_param} is
\[
\cL(\delta\theta) = R^\top\big(\pi_0 + J\delta\theta\big) - \frac{\beta}{2}\,\delta\theta^\top M\,\delta\theta ,
\]
whose first-order condition $J^\top R = \beta M\delta\theta$ gives $\delta\theta = \frac{1}{\beta}M^{-1}J^\top R$, with $\beta > 0$ determined by $\epsilon_M$. The induced step in policy space is $\pi_M - \pi_0 = J\delta\theta = \frac{1}{\beta}JM^{-1}J^\top R$.

\textbf{Reduction to a direction.}
Since $\beta > 0$, the event $\{\pi_M^\top r^* \leq \pi_0^\top r^*\}$ is $\{R^\top v_M \leq 0\}$ with
\[
v_M = JM^{-1}J^\top r^* ,
\]
using the symmetry of $M^{-1}$. This direction is deterministic and its block at $x$ is $J_xM^{-1}J^\top r^*$, so $\bm{1}^\top v_{M,x} = 0$ for every prompt because $\bm{1}^\top J_x = 0$. Hence~\eqref{eq:angle} applies and the risk equals $\Phi(-z_{\max}\cos\theta_\Sigma(v_M,y^*))$.

\textbf{Optimizing over the metric.}
Minimizing the risk is therefore maximizing $z(v)$ over the directions the parameterization can realize, i.e.\ over $v = Ju$ with $u \neq 0$:
\[
\max_{u \neq 0}\ \frac{u^\top J^\top r^*}{\sqrt{u^\top J^\top\Sigma J u}} ,
\]
a generalized Rayleigh quotient whose maximizer is $u \propto (J^\top\Sigma J)^{-1}J^\top r^*$, well defined because $J$ has full column rank and $\Sigma \succ 0$. This is the direction induced by $M = J^\top\Sigma J$, which proves the claim. Equivalently, $v_{J^\top\Sigma J}$ is the $\Sigma$-orthogonal projection of $y^*$ onto the range of $J$.

Finally, if the parameterization realizes every admissible direction, then $\mathrm{range}(J) = \cT$, which contains $y^* = \cR_\Sigma(r^*)$ by Definition~\ref{def:riesz}. The projection is then $y^*$ itself, recovering $\pi_\Sigma$.

\end{proof}

\subsection{Proof of Proposition~\ref{prop:score_form}}
\label{app:proof_scoreform}
\scoreform*

\begin{proof}
\textbf{The score is centred and almost surely non-zero.}
$\hat{S} = \cR_\Sigma(R)$ lies in $\cT$ by Definition~\ref{def:riesz}, which is the statement $\bm{1}^\top\hat{S}_x = 0$ at every prompt. It vanishes only if $R^\top u = 0$ for all $u \in \cT$, that is $R \in \cT^\perp$; as $\cT^\perp$ has dimension $\abs{\cX} < N$ and $R$ has nonsingular covariance, that event is null.

\textbf{Solving~\eqref{eq:score_objective}.}
Lemma~\ref{lem:coupling} applies to~\eqref{eq:score_objective}, whose objective vector is $G\hat{S}$ in place of $R$ and whose geometry is $G$, and gives $\pi_G - \pi_0 = \beta^{-1}\cR_G(G\hat{S})$ with $\beta > 0$ fixed by the radius, positivity being inactive by Definition~\ref{def:interior}. Now for every $u \in \cT$,
\[
\big\langle \hat{S}, u\big\rangle_G = \hat{S}^\top G u = \big(G\hat{S}\big)^\top u ,
\]
using only the symmetry of $G$. Since $\hat{S} \in \cT$, this says $\hat{S}$ satisfies the defining property~\eqref{eq:riesz} of $\cR_G(G\hat{S})$, so uniqueness gives $\cR_G(G\hat{S}) = \hat{S}$ and hence $\pi_G - \pi_0 = \beta^{-1}\hat{S}$. The direction is thus the same for every $G$; only the multiplier differs. Imposing the active constraint $\tfrac12\beta^{-2}\|\hat{S}\|_G^2 = \epsilon$ gives $\beta^{-1} = \eta_G = \sqrt{2\epsilon}/\|\hat{S}\|_G$. Uniqueness of the maximizer follows as in Lemma~\ref{lem:coupling}, the objective being linear and the feasible set the intersection of an ellipsoid with the affine normalization constraints.

\textbf{Equal risk.}
By~\eqref{eq:interior_solution_main} with $H = \Sigma$ the deviation of $\pi_\Sigma$ is also a positive multiple of $\hat{S}$, since $\cR_\Sigma(R) = \hat{S}$ by definition. For any $a > 0$,
\[
(\pi_0 + a\hat{S})^\top r^* \leq \pi_0^\top r^*
\iff
\hat{S}^\top r^* \leq 0 ,
\]
an event not depending on $a$. Taking $a = \eta_G$ and $a = \beta_\Sigma^{-1}$ shows that $\pi_G$ and $\pi_\Sigma$ underperform on exactly the same event, for any such $G$ and any pair of interior radii. By Theorem~\ref{thm:opt_geometry} its probability is $\Phi(-z_{\max})$.
\end{proof}

\section{Reward Model Training Details}
\label{sec:reward_model}

\begin{figure}[t]
\scriptsize
\begin{verbatim}
class RewardHead(PreTrainedModel):
    def __init__(self, cfg):
        super().__init__(cfg)
        self.model = Backbone(cfg)      # shared base
        self.scores = nn.ModuleList([   # m members
            nn.Linear(cfg.hidden_size, 1, bias=False)
            for _ in range(cfg.num_ensembles)])

    def forward(self, ids, mask):
        h = self.model(ids, mask).last_hidden_state
        i = mask.sum(1) - 1       # last non-pad token
        z = h[torch.arange(h.size(0)), i]
        return torch.stack(
            [s(z) for s in self.scores], dim=1)
\end{verbatim}
\caption{\footnotesize What a reward head is. The backbone maps a prompt--response pair to hidden states; the state at the final non-pad token is pooled into $z \in \mathbb{R}^{h}$, and each of the $m$ members is one bias-free linear map $z \mapsto \hat{R}_i \in \mathbb{R}$. A forward pass therefore returns $m$ scalar rewards for the pair, and $\hat{R}$ and $\sigma^2$ are the sample mean and variance across them. The members differ only in their head, so they share every nonlinearity and disagree only in how they read off the same representation. Appendix~\ref{sec:reward_model_gemma} uses $m = 10$ on a frozen \textsc{Gemma-2b-it} base, Appendix~\ref{sec:reward_model_ppo} uses $m = 20$ on a fully fine-tuned \textsc{Llama-3.2-1B} base.}
\label{fig:reward_head}
\end{figure}
\subsection{Ensemble Reward Models: Finetuning and Observations}
\label{sec:reward_model_gemma}
This subsection describes the ensemble whose reward spread is measured in Section~\ref{sec:variability}, where the disagreement of its members on a fixed prompt--response pair is what makes $\sigma^2$ observable. It places ten score heads on a shared \textsc{Gemma-2b-it} backbone \citep{team2024gemma}, an instruction-tuned variant of the foundational \textsc{Gemma-2b}; Figure~\ref{fig:reward_head} gives the construction. Training ten members on identical data and reading their variance on the same input is what isolates variability attributable to the reward model rather than to the data.

\paragraph{Dataset}
To train our reward models, we use an open-source preference dataset introduced by \citet{dong2024rlhf}, available on HuggingFace\footnote{\href{https://huggingface.co/datasets/weqweasdas/preference_dataset_mix2}{huggingface.co/weqweasdas/preference\_dataset\_mix2}}. This curated dataset comprises approximately 50{,}000 labeled preference pairs and is constructed by merging several widely-used open datasets. The constituent datasets include \textsc{HH-RLHF} \citep{bai2022training}, \textsc{SHP} \citep{ethayarajh2022understanding}, \textsc{HelpSteer} \citep{wang2023helpsteer}, \textsc{PKU-SafeRLHF} \citep{ji2024beavertails}, \textsc{UltraFeedback} \citep{cui2023ultrafeedback}, \textsc{UltraInteract} \citep{yuan2024advancing}, and \textsc{Distilabel-Orca3} \citep{lian2023openorca}.

To ensure data quality, the combined dataset is preprocessed to remove approximately 10\% of low-quality samples. The resulting dataset contains prompt–response preference pairs, where each prompt is associated with one preferred and one rejected response. These pairwise labels serve as the ground truth for training the ensemble reward models.

The diversity and breadth of the included datasets make this collection well-suited for training robust reward models that generalize across a range of domains and tasks. For further details on dataset construction and preprocessing, we refer the reader to \citet{dong2024rlhf}.

\paragraph{Methodology}  
We use the \textsc{Gemma-2b-it} model \citep{gemma-2b-it} as the backbone for our reward models. As an instruction-tuned variant of the foundational \textsc{Gemma-2b}, it is well-suited for reward modeling tasks due to its alignment with human instructions. The full model, including a scalar reward head, occupies approximately 9.34~GB on disk. Since we employ an ensemble of 10 independently parameterized reward models, the total storage requirement scales to roughly 90~GB.

To accelerate training and improve memory efficiency, we adopt the following methodology:
\begin{itemize}
    \item \textbf{Initial Training:} We begin by training a single instance of the full \textsc{Gemma-2b-it} model with a scalar reward head—a linear layer of size $2048 \times 1$—on the preference dataset. We apply early stopping once the loss reaches $0.3$ to balance model capacity and mitigate overfitting.

    \item \textbf{Parallel Reward Heads:} After partial training, we attach nine additional reward heads in parallel to the original, following the setup of \citet{zhang2024improving}. Each head has identical dimensions ($2048 \times 1$), resulting in a 10-dimensional output vector in a single forward pass—one scalar reward per ensemble member.

    \item \textbf{Freezing the Foundation Model:} To reduce computational costs, we freeze the pre-trained base of \textsc{Gemma-2b-it} and train only the reward heads. This setup simulates training 10 independent reward models in parallel while sharing the same foundation. We employ an additive loss function:
    \[
        \text{loss} = \sum_{i=1}^{10} \ell(\theta_i),
    \]
    where each $\theta_i$ denotes the parameters of the $i$-th reward head. This ensures each head is trained independently, while retaining efficiency.
\end{itemize}

By freezing the foundational layers and isolating training to the reward heads, we significantly reduce both computational and storage overhead. The final ensemble model remains compact, occupying approximately 9.34~GB on disk, with only 20,480 trainable parameters across all reward heads.

\paragraph{Hyperparameters and cost}
Both stages use an effective batch size of $32$, learning rate $10^{-5}$, paged \textsc{AdamW} in 32-bit, weight decay $0.001$ and a cosine schedule, with the remaining settings as in \citet{wolf-etal-2020-transformers}. The initial single-head finetune runs for one epoch, $4125$ steps; training the ten heads over the frozen base runs for half an epoch, $2060$ steps. Both stages use four \textsc{NVIDIA A40} GPUs of $48$~GB. The single-head stage occupies $12.68$~GB of GPU memory and $30$~GB of disk and takes $51$ hours; the heads-only stage occupies $6.12$~GB and $38$~GB and takes $7$ hours, and it is this sevenfold reduction in wall-clock time that freezing the base buys.

\paragraph{Evaluation Datasets}
The measurements reported in Section~\ref{sec:variability} and Figure~\ref{fig:reward_modelling} are taken on four datasets from the \textsc{RewardBenchmark} platform: \textsc{Chat}, \textsc{Chat-Hard}, \textsc{Safety}, and \textsc{Reasoning}. As an example, the \textsc{Chat} dataset comprises 358 prompt–response pairs of the form $(x, y^1, y^2)$, where $y^1$ is the preferred response and $y^2$ is the rejected one. This dataset is constructed from a mixture of sources, including \textsc{AlpacaEval Easy}, \textsc{AlpacaEval}, \textsc{AlpacaEval Hard}~\citep{li2023alpacaeval}, \textsc{MT Bench Easy}, and \textsc{MT Bench Medium}~\citep{zheng2023judging}. Details of the other benchmark compositions are provided in \citet{lambert2024rewardbench}. For each accepted response $y^1$ we compute the reward $r_i(x, y^1)$ under each of the $10$ reward heads and take the sample variance across the ensemble; the reward differences $r_i(x, y^1) - r_i(x, y^2)$ are treated the same way.

\begin{figure*}[t]
    \makebox[1.02\textwidth][c]{%
        \begin{minipage}{1.1\textwidth}
            \centering
            \begin{subfigure}[\footnotesize{\textbf{Chat}}]{\includegraphics[width=0.25\linewidth]{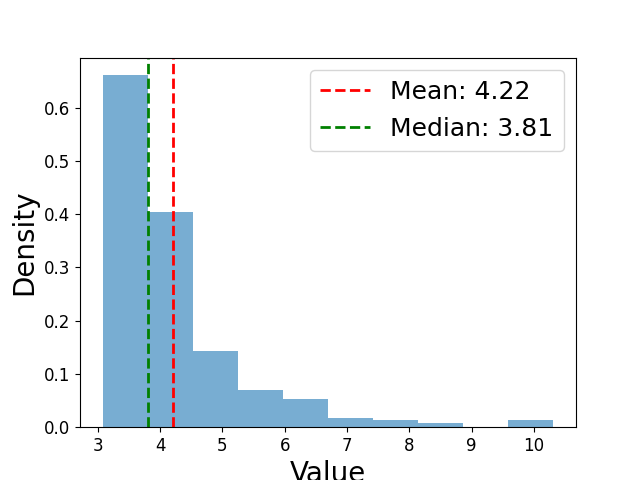}}   \end{subfigure}
            \hspace*{-0.3cm}
            \begin{subfigure}[\footnotesize{\textbf{Chat Hard}}]{\includegraphics[width=0.25\linewidth]{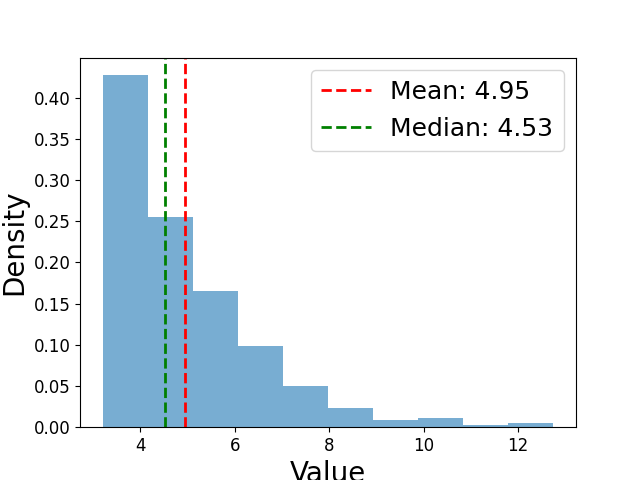}}
            \end{subfigure}
            \hspace*{-0.3cm}
            \begin{subfigure}[\footnotesize{\textbf{Safety}}]{\includegraphics[width=0.25\linewidth]{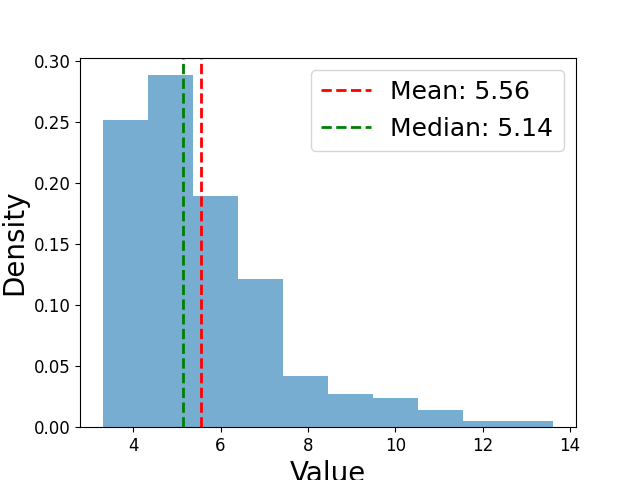}}  \end{subfigure}
            \hspace*{-0.3cm}
            \begin{subfigure}[\footnotesize{\textbf{Reasoning}}]{\includegraphics[width=0.25\linewidth]{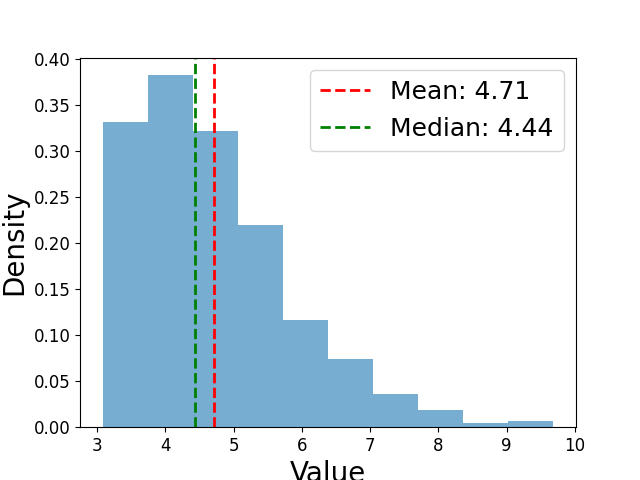}}
            \end{subfigure}
        \end{minipage}%
    }
    \caption{\footnotesize Distribution of sample variances in the rewards assigned to accepted responses, each variance computed across the outputs of 10 independently trained reward models. The median sample variance exceeds $3.81$ in all datasets, with values reaching as high as $10$, confirming substantial model-level variability in reward estimates for the same prompt–response pair. Figure~\ref{fig:reward_diff} repeats the measurement on reward differences.}
    \label{fig:reward_modelling}
\end{figure*}

\begin{figure*}[t]
    \makebox[1.02\textwidth][c]{%
        \begin{minipage}{1.1\textwidth}
            \centering
            \begin{subfigure}[\footnotesize{\textbf{Chat}}]{\includegraphics[width=0.25\linewidth]{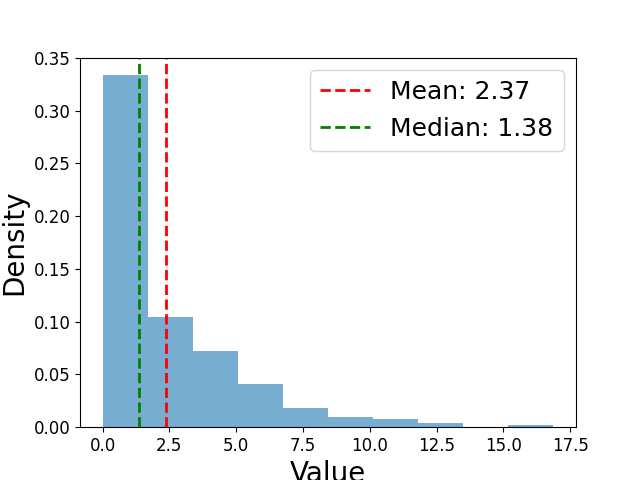}}   \end{subfigure}
            \hspace*{-0.3cm}
            \begin{subfigure}[\footnotesize{\textbf{Chat Hard}}]{\includegraphics[width=0.25\linewidth]{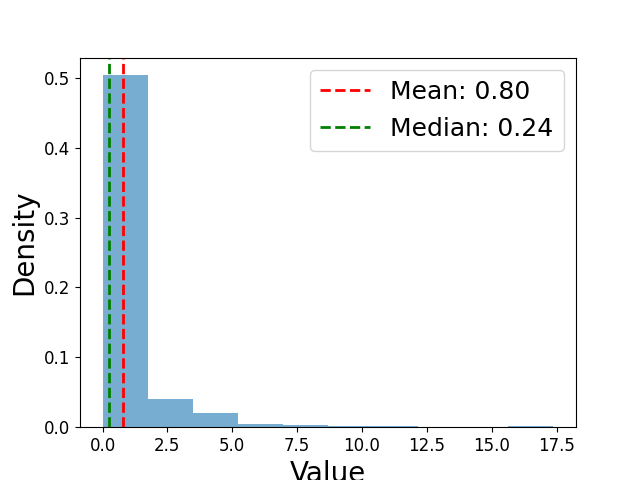}}
            \end{subfigure}
            \hspace*{-0.3cm}
            \begin{subfigure}[\footnotesize{\textbf{Safety}}]{\includegraphics[width=0.25\linewidth]{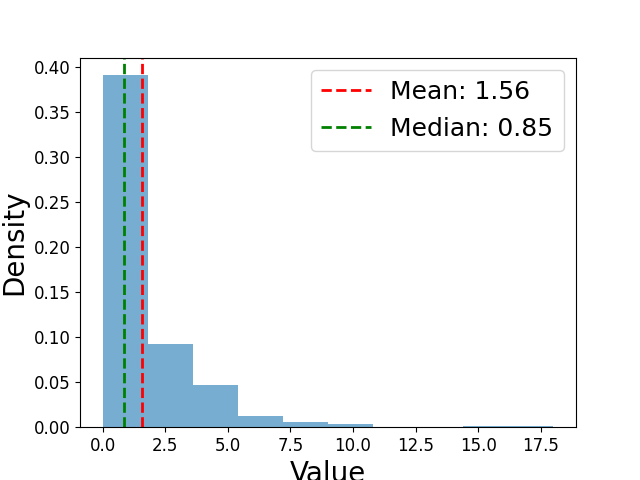}}  \end{subfigure}
            \hspace*{-0.3cm}
            \begin{subfigure}[\footnotesize{\textbf{Reasoning}}]{\includegraphics[width=0.25\linewidth]{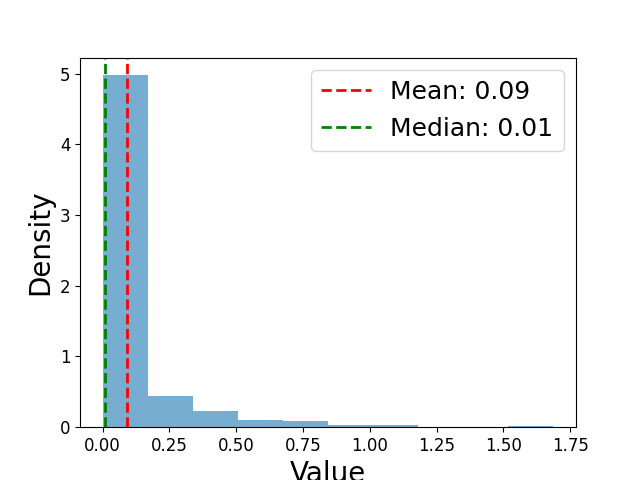}}
            \end{subfigure}
        \end{minipage}%
    }
    \caption{\footnotesize Distribution of the sample variances of reward differences between accepted and rejected responses, computed across the 10 independently trained reward models. Were the members simple scalar translations of one another these differences would agree exactly and the distributions would collapse to a point mass at zero. The spread instead reflects genuine disagreement about preferences, driven by the statistical and stochastic nature of training.}
    \label{fig:reward_diff}
\end{figure*}

\subsection{Ensemble Reward Model for the PPO Experiments}
\label{sec:reward_model_ppo}

The ensemble that supplies $\hat{R}$ and $\sigma^2$ in Section~\ref{sec:llm} is built the same way as the one in Appendix~\ref{sec:reward_model_gemma}: independent linear score heads read off a shared backbone at the last non-pad token, and all heads are trained jointly under the additive Bradley--Terry loss $\sum_i \ell(\theta_i)$ with $\ell(\theta_i) = -\log\sigma\big(\hat{R}_i(x,y^+) - \hat{R}_i(x,y^-)\big)$, so each head is an independent reward model over a common base. It differs in the following respects.

\paragraph{Backbone and ensemble size}
Twenty score heads rather than ten, on a \textsc{Llama-3.2-1B} base \citep{llama_cite} of sixteen layers and hidden width $2048$, so each head is a linear map $2048 \times 1$ as in Figure~\ref{fig:reward_head}. We supervised fine-tune \texttt{meta-llama/Llama-3.2-1B} on \textsc{UltraChat-200k} \citep{ding2023ultrachat} and use that checkpoint as the base, rather than the released instruction-tuned variant. Preference training then updates all sixteen layers.
\paragraph{Dataset}
The \texttt{alpaca\_human\_preference} configuration of \textsc{AlpacaFarm} \citep{dubois2023alpacafarm}: $9{,}691$ examples of an instruction with two candidate outputs and a human preference label, split $95/5$ into training and held-out evaluation. This is the same instruction distribution on which the policies of Section~\ref{sec:ppo_experiment} are fine-tuned, and is far smaller than the $50{,}000$-pair mixture of Appendix~\ref{sec:reward_model_gemma}.

\paragraph{Training}
The base is fully fine-tuned rather than frozen, which is the one departure from the procedure of Appendix~\ref{sec:reward_model_gemma}. Learning rate $3 \times 10^{-6}$, weight decay $0.01$, cosine schedule, \texttt{bf16}, per-device batch size $16$ with two gradient accumulation steps over four-way data parallelism, for an effective batch of $128$. Training stops early at epoch $11$ of $20$, once the ensemble accuracy has not improved over eight consecutive evaluations; the checkpoint used throughout is the one from epoch $7$.

\paragraph{Held-out performance}
On the $485$ evaluation pairs the loss summed over the twenty heads is $13.59$, and the mean per-head preference accuracy is $58.6\%$.

\section{Supplementary Details for the Bandit Experiments}
\label{sec:bandits_supp}

This appendix expands the methodology and observations of Section~\ref{sec:bandits}.

\paragraph{Methodology.}
A bandit with $K$ arms plays the role of one prompt with $K$ candidate responses. We fix a reference policy $\pi_0$, a true reward $r^* \in \mathbb{R}^K$ and a diagonal $\Sigma$, and draw reward models $\hat{R} \sim \cN(r^*, \Sigma)$. The variances are drawn log-uniformly from $[\sigma^2_{\min}, \kappa\,\sigma^2_{\min}]$ with $\sigma^2_{\min} = 0.1$, so that $\kappa$ is the ratio between the largest and smallest variance a pair can receive, a condition number for $\Sigma$. We use $\kappa = 50$ unless stated otherwise, and $\kappa = 1$ means every pair carries the same variance. Each draw is optimized by every method, and the underperformance probability is the fraction of draws with $\pi^\top r^* \leq \pi_0^\top r^*$. We compare the four instances of~\eqref{eq:trust_region} named in Section~\ref{sec:geometry}, namely vanilla ($H = I$), covariance ($H = \Sigma$), KL ($H = \Pi_0^{-1}$) and variance-weighted KL ($H = \Pi_0^{-1}\Sigma$), all scored by $\hat{R}$, together with the uncertainty-penalty baseline UWO$(\lambda)$, which scores $\hat{R} - \lambda\sigma^2$ in the unweighted geometry. Each problem is solved exactly, including $\pi \geq 0$, through the KKT system of Lemma~\ref{lem:kkt}, at that geometry's interior radius unless stated otherwise. In the interior regime the underperformance event is $\{\hat{R}^\top v_H \leq -b^\top v_H\}$ for a deterministic direction $v_H$ and score shift $b$, so the risk has the closed form $\Phi(-(v_H^\top r^* + b^\top v_H)/\|v_H\|_\Sigma)$. We report that value and use Monte Carlo over sampled $\hat{R}$ as a check. Results average over independent instances of $(\pi_0, r^*, \Sigma)$: $10$ instances and $2000$ draws for Tables~\ref{tab:bandit_arms} and~\ref{tab:bandit_skew}, $20$ and $4000$ for Tables~\ref{tab:bandit_radius}--\ref{tab:bandit_param}. At a fixed per-arm reward scale $z$ grows with $K$, driving every risk to zero, so we normalize $\|r^*\| = 1$ to hold the difficulty fixed as the dimension grows.

\paragraph{Observations.}
The covariance geometry is the safest throughout, by about a factor of three. On a uniform reference it underperforms on $5$ to $7\%$ of reward models against $16$ to $19\%$ for the unweighted baseline, a gap stable from $K = 50$ to $K = 1000$ (Table~\ref{tab:bandit_arms}). UWO$(0.1)$ only matches the baseline and UWO$(1.0)$ roughly doubles its risk, consistent with its shifting the score rather than reshaping the trust region. Closed-form and Monte-Carlo estimates agree to within $0.004$, which validates the Gaussian reduction of Section~\ref{sec:geometry} and confirms that the nonnegativity constraints are inactive at an interior radius. Two pairs of rows in Table~\ref{tab:bandit_arms} coincide to every digit, and exactly so: for uniform $\pi_0$ we have $\Pi_0^{-1} = K I$ and $\Pi_0^{-1}\Sigma = K\Sigma$, and the solution of~\eqref{eq:trust_region} depends on $H$ only up to positive scaling, so KL is vanilla and variance-weighted KL is the covariance geometry. At $K = 2$ all methods sharing a score agree, since $\{y : \bm{1}^\top y = 0\}$ is then one-dimensional. That coincidence is fragile in the way Section~\ref{sec:geometry} predicts. With $\pi_0 \propto (i+1)^{-\alpha}$ the $\ell_2$ and covariance geometries are invariant to $\alpha$, having no dependence on $\pi_0$, whereas both KL geometries decay monotonically, until at $\alpha = 2$ the variance-weighted KL ($0.374$) is worse than the unweighted baseline ($0.194$) and close to plain KL ($0.414$), its margin $z$ having fallen from $1.58$ to $0.32$ (Table~\ref{tab:bandit_skew}). Sweeping the radius past the interior threshold reproduces Remark~\ref{rem:boundary} quantitatively (Table~\ref{tab:bandit_radius}): up to $\epsilon_0$ nothing is clipped and measurement matches the closed form to three decimals, while beyond it the optimum falls on the boundary for almost every draw and the closed form becomes badly wrong, predicting $0.035$ where the measured risk is $0.099$. The guarantee is then lost rather than merely loosened: at $200\epsilon_0$ the covariance geometry is the riskier of the two on one instance in twenty, an ordering that Theorem~\ref{thm:opt_geometry} rules out inside the interior regime. The size of the gain depends on the uncertainty it exploits (Table~\ref{tab:bandit_sigma}): the geometries are indistinguishable at $\kappa = 1$, where every pair carries the same variance, and separate steadily as $\kappa$ grows. Replacing the exact $\Sigma$ by an estimate from $m$ ensemble members degrades gracefully down to $m = 5$, but at $m = 3$ gives $0.052$, worse than the unweighted baseline's $0.034$, so too noisy a variance estimate is more harmful than none. Finally, confining updates to the range of a random $J \in \mathbb{R}^{K \times d}$ with $\bm{1}^\top J = 0$ leaves the covariance pullback $J^\top\Sigma J$ safest at every rank and, at full rank, equal to the unrestricted optimum to four decimals, as Proposition~\ref{prop:parametric} requires. The gap at smaller $d$ is the price of projecting onto what the parameterization can express, and it is large: at $d = 2$ every metric sits near $0.15$, against the $0.008$ available when all directions are reachable (Table~\ref{tab:bandit_param}).

\subsection{Variance Heterogeneity and Estimated Covariance}
\label{sec:bandit_sigma}
These two sweeps isolate what the covariance geometry has to exploit: how unequal the reward variances are, and how well $\Sigma$ itself is estimated. Panel~(a) of Table~\ref{tab:bandit_sigma} varies $\kappa = \sigma^2_{\max}/\sigma^2_{\min}$, the ratio between the largest and smallest variance an arm can receive. At $\kappa = 1$ every arm carries the same variance, $\Sigma$ is a multiple of the identity, and the covariance geometry is by construction the unweighted one, so the two agree exactly; the gap then widens monotonically, reaching $0.024$ against $0.105$ at $\kappa = 200$. The advantage is therefore not a property of preconditioning as such but of the heterogeneity it has to work with, which is the quantity Section~\ref{sec:candidate_experiment} finds to be small on real candidate sets.

Panel~(b) replaces the exact $\Sigma$ by an estimate from $m$ ensemble members at $\kappa = 50$. Degradation is graceful down to $m = 5$, which still gives $0.023$ against the unweighted baseline's $0.034$, but at $m = 3$ the estimated method gives $0.052$ and is worse than ignoring uncertainty altogether. A preconditioner is only as good as the second moment it is built from: too few members and its own sampling noise dominates the signal it is meant to exploit, which inverts the comparison the theorem predicts.

\begin{table}[H]
\centering
\footnotesize
\setlength{\tabcolsep}{4pt}
\begin{tabular}{lccccc}
\toprule
\multicolumn{6}{c}{(a) variance ratio $\kappa$, exact $\Sigma$} \\
\midrule
$\kappa$   & $1$ & $3$ & $10$ & $50$ & $200$ \\
vanilla           & $0.000$ & $0.000$ & $0.004$ & $0.034$ & $0.105$ \\
covariance        & $0.000$ & $0.000$ & $\bm{0.001}$ & $\bm{0.008}$ & $\bm{0.024}$ \\
\midrule
\multicolumn{6}{c}{(b) members $m$ estimating $\Sigma$, at $\kappa = 50$} \\
\midrule
$m$               & $3$ & $5$ & $10$ & $20$ & exact \\
covariance        & $0.052$ & $0.023$ & $0.013$ & $0.010$ & $\bm{0.008}$ \\
vanilla reference & $0.034$ & $0.034$ & $0.034$ & $0.034$ & $0.034$ \\
\bottomrule
\end{tabular}
\caption{\footnotesize (a) Risk against the heterogeneity of the reward variances, at an interior radius, where $\kappa = \sigma^2_{\max}/\sigma^2_{\min}$ is the ratio between the largest and smallest variance a pair can receive. (b) The same when $\Sigma$ is estimated from $m$ members rather than known; at $m = 3$ the estimated method is worse than ignoring uncertainty altogether.}
\label{tab:bandit_sigma}
\end{table}

\subsection{Restricted Parameterizations}
\label{sec:bandit_extra}
Table~\ref{tab:bandit_param} confines updates to the range of a random $J \in \mathbb{R}^{K\times d}$ with $\bm{1}^\top J = 0$, so that only $d$ of the $K-1$ tangent directions are reachable, and compares the three metrics $M$ of~\eqref{eq:trust_region_param}. The covariance pullback $J^\top\Sigma J$ is safest at every rank, and at $K = 8$ with $d = 7$, where $J$ is full rank, it attains $0.008$, equal to the unrestricted optimum of Theorem~\ref{thm:opt_geometry} to four decimals, as Proposition~\ref{prop:parametric} requires. The price of a restricted parameterization is paid at smaller $d$ and it is steep: at $d = 2$ every metric sits near $0.15$, against the $0.008$ available when all directions are reachable, so what the parameterization can express matters more than which metric is used within it.

\begin{table}[H]
\centering
\footnotesize
\begin{tabular}{lccc}
\toprule
& \multicolumn{3}{c}{rank $d$ of the parameterization} \\
\cmidrule(lr){2-4}
metric $M$ & $2$ & $4$ & $7$ \\
\midrule
$I$ (Euclidean)                & $0.177$ & $0.123$ & $0.092$ \\
$J^\top\Pi_0^{-1}J$ (Fisher)   & $0.158$ & $0.078$ & $0.030$ \\
$J^\top\Sigma J$ (covariance)  & $\bm{0.150}$ & $\bm{0.057}$ & $\bm{0.008}$ \\
\midrule
unrestricted optimum           & $0.008$ & $0.008$ & $0.008$ \\
\bottomrule
\end{tabular}
\caption{\footnotesize Risk for a policy whose updates are confined to the range of $J$, at $K = 8$, so $d = 7$ is full rank, where the covariance pullback coincides with the unrestricted optimum of Theorem~\ref{thm:opt_geometry}.}
\label{tab:bandit_param}
\end{table}

\section{Supplementary Details for the LLM Experiments}
\label{sec:llm_supp}

\paragraph{Reward draws, the variance estimate and the objectives.}
Theorem~\ref{thm:opt_geometry} is a statement about the probability, \emph{over draws of the reward model}, that the fine-tuned policy is worse than the reference policy under the true reward. The experiments therefore vary the reward model rather than the training seed. Each of the $m$ heads of the ensemble is one draw of $\hat{R}$: a run uses head $k$ alone as the reward and estimates $\sigma^2(x,y)$ from the leave-one-out sample variance of the other $m-1$ heads,
\begin{equation}
\label{eq:loo_variance}
\begin{aligned}
\hat{\sigma}^2 &:= \frac{1}{m-2}\sum_{j \neq k}\big(\hat{R}_j - \overline{R}_{-k}\big)^2 , \\
\text{where}\quad \overline{R}_{-k} &:= \frac{1}{m-1}\sum_{j\neq k}\hat{R}_j ,
\end{aligned}
\end{equation}
formed at each prompt--response pair with $m = 20$. Holding out head $k$ makes the uncertainty estimate independent of the reward it corrects, so the preconditioner cannot inherit the errors of that reward, and this $\hat{\sigma}^2$ is what the score~\eqref{eq:precondition} is built from. Three objectives are compared throughout, and they share the draw index so that every comparison is paired: \emph{vanilla} carries $\hat{R}$; \emph{covariance} carries the score $\hat{S} = (\hat{R} - c)/\sigma^2$ of~\eqref{eq:precondition}, with $c = \sum_j \hat{R}_j/\sigma_j^2 \big/ \sum_j 1/\sigma_j^2$ the precision-weighted mean over the batch; and \emph{UWO}$(\lambda)$ carries $\hat{R} - \lambda\sigma^2$ at $\lambda = 0.1$ \citep{coste2023reward}. In Section~\ref{sec:ppo_experiment} all three keep PPO's KL penalty, so the geometry is held fixed and the objective is the only difference between arms.

\paragraph{Models, reward models and data.}
Every experiment uses the \textsc{AlpacaFarm} instruction data \citep{dubois2023alpacafarm}. The proxy reward model is the twenty-head \textsc{Llama-3.2-1B} ensemble of Appendix~\ref{sec:reward_model_ppo}, whose base is supervised fine-tuned on \textsc{UltraChat-200k} before preference training on \textsc{AlpacaFarm} human preferences; one forward pass returns twenty rewards for the same input. The gold reward $r^*$ is the \textsc{AlpacaFarm} human reward model, used only for evaluation and never for training. Every policy is supervised fine-tuned before any of these experiments: Section~\ref{sec:candidate_experiment} generates candidates from a supervised \textsc{Qwen2.5-0.5B}, and Section~\ref{sec:ppo_experiment} fine-tunes a supervised \textsc{Pythia-1.4B} \citep{biderman2023pythia} and a supervised \textsc{Qwen2.5-7B} \citep{Yang2024Qwen25TR} with PPO.

\subsection{Candidate-Set Protocol}
\label{sec:candidate_protocol}

Candidates are drawn at temperature $1$ with no truncation, to at most $256$ new tokens, $K = 16$ per prompt over $512$ prompts. The language model is used only to generate the candidates and to supply their sequence log-probabilities; no parameters are trained.

The reference policy of Section~\ref{sec:candidate_experiment}, $\pi_0(y_j\mid x) \propto p_j^{\,\log\alpha/\Delta_x}$ with $\Delta_x := \max_j \log p_j - \min_j \log p_j$, is the tempered softmax $\exp(\log p_j/T_x)$ at the per-prompt temperature
\[
T_x = \Delta_x\big/\log\alpha .
\]
Its largest and smallest reference probabilities are in the ratio $p_{\max}^{\log\alpha/\Delta_x}/p_{\min}^{\log\alpha/\Delta_x} = \alpha$ at every prompt, and $\alpha \to 1$ sends the exponent to zero and hence $\pi_0$ to the uniform policy. The temperature has to be set per prompt: candidate log-probabilities span hundreds of nats, so any fixed $T$ either leaves $\pi_0$ uniform or collapses it onto a single candidate.

Each geometry is solved at half its own interior threshold, $\epsilon = \tfrac12\cdot\tfrac12(\min_x\rho(x)\lambda_{\min}(H_x))p_{\min}^2$, so the radii differ between geometries, which Theorem~\ref{thm:opt_geometry} permits. We check $\pi > 0$ coordinatewise at every one of the $10{,}240$ solutions and record no violation, so the closed form~\eqref{eq:interior_solution_main} applies throughout and no positivity constraint is active.

\paragraph{The measured test.}
The same solve is first carried out with the ensemble itself supplying the reward. A sample is indexed by a prompt $x$ and a head $k \in \{1,\dots,20\}$. The reward is $\hat{R}_{x,k} \in \mathbb{R}^{16}$, the scores head $k$ assigns to the sixteen candidates of prompt $x$; the variance $\sigma^2_{x,k} \in \mathbb{R}^{16}$ is the leave-one-out sample variance~\eqref{eq:loo_variance} of the remaining nineteen heads on those same candidates, floored at the tenth percentile of the positive variances over the whole set so that a near-zero entry cannot dominate a solve; and the true reward $r^*_x$ is the vector of gold scores. Each pairing of an objective with a geometry is solved exactly at its interior radius, and the sample contributes the indicator $\mathbf{1}\{(\pi - \pi_0)^\top r^*_x \leq 0\}$. Each reported quantity is then the average over all $512 \times 20 = 10{,}240$ prompt--head samples. 

\subsection{Protocol for the PPO Experiments}
\label{sec:ppo_details}

\paragraph{How the risk is computed.}
We train one policy for each ensemble head $k \in \{1,\dots,20\}$ and each KL target $\tau$. Head $k$ is the reward $\hat{R}$, and the PPO seed is fixed. Every policy is evaluated at steps $0, 250, \dots, 3000$ on the same $128$ held-out prompts. At each checkpoint $s$ we record two numbers: the mean gold reward $g_{k,\tau}(s)$ over those prompts, and the mean KL divergence $\mathrm{kl}_{k,\tau}(s)$ from the reference policy.

The radius $\epsilon$ is a separate index from the target. The targets are $\tau \in \{0.25, 0.5, 1, 2, 4, 8\}$ and the radii are the columns of Table~\ref{tab:ppo_policy_scaling}, $\epsilon \in \{0.5, 1, 2, 4, 8\}$. A target is what the controller is asked to hold; a radius is where we read the policy. The two differ because the controller settles the penalty coefficient rather than the divergence. One run therefore contributes at every radius it reaches, whatever its target.

At radius $\epsilon$ we read a run at its last checkpoint inside that radius,
\[
s^*_{k,\tau}(\epsilon) := \max\{\,s > 0 : \mathrm{kl}_{k,\tau}(s) \leq \epsilon\,\} ,
\]
and record one indicator,
\[
I_{k,\tau}(\epsilon) := \mathbf{1}\big\{\,g_{k,\tau}\big(s^*_{k,\tau}(\epsilon)\big) \leq g_{k,\tau}(0)\,\big\} .
\]
It equals one when the trained policy is no better than the reference policy under the gold reward.

Let $A(\epsilon)$ be the set of pairs $(k,\tau)$ at which every objective has such a checkpoint. The entry reported at radius $\epsilon$ is the sample average over that set,
\[
\widehat{P}(\epsilon) \;=\; \frac{1}{|A(\epsilon)|}\sum_{(k,\tau)\,\in\,A(\epsilon)} I_{k,\tau}(\epsilon) .
\]
Table~\ref{tab:ppo_samples} gives $|A(\epsilon)|$ for every target and radius. The totals are $76$ samples for \textsc{Pythia-1.4B} and $20$ for \textsc{Qwen2.5-7B}, falling to $68$ and $18$ at $\epsilon = 0.5$. That resolves a probability to about $0.06$ and $0.1$ respectively.

Every objective is averaged over the same set $A(\epsilon)$. We therefore compare them by McNemar's exact test on the indicators, and by a sign test on the margins $g_{k,\tau}(s^*_{k,\tau}(\epsilon)) - g_{k,\tau}(0)$. The margin test is insensitive to $g_{k,\tau}(0)$, which is a single pass of the reference policy over the $128$ prompts and cancels in a difference of margins.

\begin{table}[!h]
\centering
\footnotesize
\setlength{\tabcolsep}{5pt}
\begin{tabular}{lccccc}
\toprule
& \multicolumn{5}{c}{radius $\epsilon$} \\
\cmidrule(lr){2-6}
target $\tau$ & $0.5$ & $1$ & $2$ & $4$ & $8$ \\
\midrule
\multicolumn{6}{c}{\textbf{\textsc{Pythia-1.4B}}} \\
$0.25$ & $15$ & $15$ & $15$ & $15$ & $15$ \\
$0.5$ & $20$ & $20$ & $20$ & $20$ & $20$ \\
$1$ & $11$ & $15$ & $15$ & $15$ & $15$ \\
$2$ & $12$ & $12$ & $12$ & $12$ & $12$ \\
$4$ & $7$ & $9$ & $9$ & $9$ & $9$ \\
$8$ & $3$ & $5$ & $5$ & $5$ & $5$ \\
\midrule
all targets & $\bm{68}$ & $\bm{76}$ & $\bm{76}$ & $\bm{76}$ & $\bm{76}$ \\
\midrule
\multicolumn{6}{c}{\textbf{\textsc{Qwen2.5-7B}}} \\
$0.5$ & $ 18 $ & $20$ & $20$ & $20$ & $20$ \\
$1$ & $16$ & $15$ & $17$ & $17$ & $18$ \\
$2$ & $15$ & $12$ & $12$ & $13$ & $15$ \\
$4$ & $7$ & $10$ & $10$ & $11$ & $11$ \\
$8$ & $3$ & $5$ & $5$ & $5$ & $5$ \\
\midrule
all targets & $\bm{60}$ & $\bm{62}$ & $\bm{64}$ & $\bm{66}$ & $\bm{69}$ \\
\bottomrule
\end{tabular}
\caption{\footnotesize Number of samples $|A(\epsilon)|$ averaged at each radius, broken down by KL target. A sample is one ensemble head trained at one target, counted only when every objective has a checkpoint at or below that radius. Counts fall at the tightest radius because a run trained at a loose target may never come that close to the reference policy, and fall with the target for the same reason. The bold rows are the totals that Table~\ref{tab:ppo_policy_scaling} reports. A seventh target $\tau = 6$ was trained but contributes nothing, the three objectives having been run there on disjoint heads.}
\label{tab:ppo_samples}
\end{table}

\paragraph{Training routine.}
A run trains the policy against head $k$ under the objective's score and PPO's KL penalty. Each step generates rollouts from the current policy, scores them with the ensemble, forms $\hat{\sigma}^2$ by~\eqref{eq:loo_variance} and the score by~\eqref{eq:precondition}, and takes clipped surrogate steps with generalized advantage estimation. For the covariance arm the preconditioned scores are rescaled to the norm of the centred raw scores, which leaves their direction unchanged and keeps the score-to-KL balance comparable with the other arms; variances are floored at the tenth percentile of a rolling window over the last $256$ scored samples, so that a near-zero variance cannot dominate a batch.

\subsection{Reward Model Size}
\label{sec:reward_scale}

We next evaluate the sensitivity of the methods to the size of the reward-model ensemble used to estimate uncertainty. We repeat the \textsc{Qwen2.5-7B} experiment from Table~\ref{tab:ppo_policy_scaling} using $m\in\{5,10,20\}$ reward models to estimate the reward variance, while keeping the policy, trust-region radii, and other experimental settings fixed. Table~\ref{tab:ppo_reward_scaling} shows that the covariance objective remains the lowest-risk method at every radius and for every ensemble size. Its risk increases as the ensemble becomes smaller, from $0.047$--$0.070$ with $20$ models to $0.145$--$0.215$ with $5$ models, but it continues to substantially outperform both Vanilla PPO and UWO$(0.1)$. In contrast, UWO$(0.1)$ remains nearly indistinguishable from Vanilla PPO across all ensemble sizes and radii.

\begin{table*}[t]
\centering
\footnotesize
\setlength{\tabcolsep}{5pt}
\begin{tabular}{llccccc}
\toprule
& & \multicolumn{5}{c}{trust-region radius} \\
\cmidrule(lr){3-7}
$m$ & Objective & $0.5$ & $1$ & $2$ & $4$ & $8$ \\
\midrule
$20$ & Vanilla & $0.277$ & $0.363$ & $0.472$ & $0.553$ & $0.643$ \\
 & UWO$(0.1)$ & $0.267$ & $0.338$ & $0.462$ & $0.535$ & $0.623$ \\
 & Covariance & $\bm{0.070}$ & $\bm{0.055}$ & $\bm{0.058}$ & $\bm{0.052}$ & $\bm{0.047}$ \\
\midrule
$10$ & Vanilla & $0.427$ & $0.510$ & $0.572$ & $0.665$ & $0.723$ \\
 & UWO$(0.1)$ & $0.417$ & $0.497$ & $0.567$ & $0.658$ & $0.712$ \\
 & Covariance & $\bm{0.170}$ & $\bm{0.145}$ & $\bm{0.125}$ & $\bm{0.103}$ & $\bm{0.107}$ \\
\midrule
$5$ & Vanilla & $0.482$ & $0.543$ & $0.608$ & $0.695$ & $0.740$ \\
 & UWO$(0.1)$ & $0.478$ & $0.540$ & $0.603$ & $0.692$ & $0.740$ \\
 & Covariance & $\bm{0.215}$ & $\bm{0.197}$ & $\bm{0.153}$ & $\bm{0.142}$ & $\bm{0.145}$ \\
\bottomrule
\end{tabular}
\caption{\footnotesize Reward model scalability. The construction of Table~\ref{tab:ppo_policy_scaling} repeated on \textsc{Qwen2.5-7B} at three ensemble sizes of $m$. The covariance arm is the safest at every radius and every ensemble size, and the uncertainty penalty stays within Monte-Carlo error of vanilla throughout. Bold marks the lowest risk in a column. Remaining parameters as in Table~\ref{tab:ppo_policy_scaling}.}
\label{tab:ppo_reward_scaling}
\end{table*}

\subsection{Hyperparameters and Training Parameters}
\label{sec:ppo_hparams}

Each run performs $3000$ PPO steps on \textsc{AlpacaFarm} prompts at learning rate $10^{-6}$, with four PPO epochs per rollout, four rollouts per step, generations of at most $256$ new tokens at temperature $1$, clip range $0.2$ and generalized advantage estimation parameter $\lambda = 0.95$. The \textsc{Pythia-1.4B} policy is trained in \textsc{bf16} with the last two transformer blocks trainable and the remainder frozen. The runs that supply the action space sweep the KL target over $0.25$ to $8$ with twenty reward draws per objective: $411$ \textsc{Pythia-1.4B} runs and $40$ \textsc{Qwen2.5-7B} runs contribute evaluation generations, at roughly one hour per run on \textsc{MI300X} accelerators.

\subsection{The Difficulty of Implementation}
\label{sec:implementability}

Proposition~\ref{prop:score_form} pairs $\hat{S}$ with an $\ell_2$ trust region and $\hat{S}/(\rho\pi_0)$ with a KL one. Neither is available in a policy-gradient pipeline over sequences. The first three obstructions below follow from $\pi_0(x,y)$ being a product of a few hundred token probabilities, so that $\log\pi_0$ runs to several hundred nats.

\textbf{The $\ell_2$ constraint underflows.} Its constraint is $\tfrac12\sum_{x,y}(\pi(x,y)-\pi_0(x,y))^2$, a sum over all pairs, whereas a pipeline sees only samples $y\sim\pi_0(\cdot\mid x)$. Written in the importance ratio $r(x,y) := \pi(y\mid x)/\pi_0(y\mid x)$ that PPO already forms, the deviation is $\pi_0(x,y)\,(r(x,y)-1)$ and the constraint becomes
\[
\tfrac12\,\mathbb{E}_{y\sim\pi_0}\big[\pi_0(x,y)\,(r(x,y)-1)^2\big] ,
\]
so each sampled term carries the factor $\pi_0(x,y) = \exp(\log\pi_0)$. Single precision has smallest normal $2^{-126} = \exp(-87.3)$, so for sequences of even a hundred tokens this factor is exactly zero in floating point and the constraint has no gradient.

\textbf{The interior radius is unreachable.} Definition~\ref{def:interior} requires $2\epsilon < \lambda_{\min}(G)\,q_{\min}^2$ with $q_{\min} = \min_{x,y}\pi_0(x,y)$. At sequence scale $q_{\min}^2 = \exp(2\log q_{\min})$ is smaller than any representable radius, so the positivity constraints are active at every usable $\epsilon$ and the closed form~\eqref{eq:interior_solution_main} does not apply. This is the regime of Remark~\ref{rem:boundary}, not of Theorem~\ref{thm:opt_geometry}.

\textbf{The reweighted score has unbounded variance.} The KL pairing carries $w(x,y) = \hat{S}(x,y)/(\rho(x)\pi_0(x,y))$, dividing the score by the sequence probability. Across sixteen distinct responses to one prompt, taken from the evaluation generations, the reference log-probability spans a median of $56$ nats (minimum $7.5$, maximum $161$ over $97$ prompts), so $1/\pi_0$ varies by a factor of order $e^{56}$ within a single prompt. The sampled objective is then determined by whichever response the reference policy finds least likely.

\textbf{The per-prompt baseline is not estimable.} Both the theory and~\eqref{eq:precondition} centre each prompt by its own variance-weighted mean, $c^*_x$ in~\eqref{eq:ystar} and $c_x$ in~\eqref{eq:precondition}. With a single response drawn per prompt that quantity equals the one sampled reward, the centred score vanishes and no update is made. We therefore pool prompts and use one baseline per batch, which recovers $c_x$ only when $\hat{\sigma}^2$ and $r^*$ vary little across the prompts of a batch. Drawing several responses per prompt would close this gap at proportional cost.

\textbf{What is implemented instead.} Pairing $\hat{S}$ with PPO's own KL penalty realizes $\Pi_0\hat{S}$ rather than $\hat{S}$, which is $\hat{S}/(\rho\pi_0)$ clipped to $\hat{S}$. The preconditioner also divides by $\hat{\sigma}^2$, an $m$-member estimate, rather than by the covariance itself. Together these forfeit the equality of Proposition~\ref{prop:score_form} and leave the objective as the only difference between arms, which is what the comparison in Section~\ref{sec:ppo_experiment} isolates. Section~\ref{sec:candidate_experiment} measures what the substitution costs on an action space where the exact pairings can be solved.

\end{document}